%% file: main.tex
\documentclass[acmsmall]{acmart}

\usepackage{algorithmic}
\usepackage{array}
\usepackage{lipsum}
\usepackage{hyperref}
\usepackage{multirow}
\usepackage{makecell}
\usepackage{xcolor}
\usepackage{enumerate}
\usepackage{tikz}

\usepackage{ulem}
\usepackage{etoolbox}
\usepackage{enumitem}
\newcommand*{\circled}[1]{\lower.7ex\hbox{\tikz\draw (0pt, 0pt)%
    circle (.5em) node {\makebox[1em][c]{\small #1}};}}
\robustify{\circled}

\AtBeginDocument{%
  \providecommand\BibTeX{{%
    \normalfont B\kern-0.5em{\scshape i\kern-0.25em b}\kern-0.8em\TeX}}}

\setcopyright{acmcopyright}
\copyrightyear{2022}
\acmYear{2022}

\input{commands}

\begin{document}

\title{No free lunch theorem for security and utility in federated learning}

%
\author{Xiaojin Zhang}
\authornote{Both authors contributed equally to this research.}
\email{xiaojinzhang@ust.hk}
\affiliation{%
  \institution{Hong Kong University of Science and Technology}
  \streetaddress{Clear Water Bay}
  \country{Hong Kong}
}

\author{Hanlin Gu}
\authornotemark[1]
\email{allengu@webank.com}
\affiliation{%
  \institution{Webank}
  \city{Shenzhen}
  \country{China}
}

\author{Lixin Fan}
\email{lixinfan@webank.com}
\affiliation{%
  \institution{Webank}
  \city{Shenzhen}
  \country{China}
}

\author{Kai Chen}
\email{kaichen@cse.ust.hk}
\affiliation{%
  \institution{Hong Kong University of Science and Technology}
  \streetaddress{Clear Water Bay}
  \country{Hong Kong}
}

\author{Qiang Yang}
\email{qyang@cse.ust.hk}
\authornote{Corresponding author}
\affiliation{%
  \institution{WeBank and Hong Kong University of Science and Technology}
  \country{Hong Kong}
}






%
\renewcommand{\shortauthors}{Trovato and Tobin, et al.}

\begin{abstract}

In a federated learning scenario where multiple parties jointly learn a model from their respective data, there exist two conflicting goals for the choice of appropriate algorithms. On one hand, private and sensitive training data must be kept secure as much as possible in the presence of \textit{semi-honest} partners, while on the other hand, a certain amount of information has to be exchanged among different parties for the sake of learning utility. Such a challenge calls for the privacy-preserving federated learning solution, which maximizes the utility of the learned model and maintains a provable privacy guarantee of participating parties' private data.

This article illustrates a general framework that a) formulates the trade-off between privacy loss and utility loss from a unified information-theoretic point of view, and b) delineates quantitative bounds of privacy-utility trade-off when different protection mechanisms including Randomization, Sparsity, and Homomorphic Encryption are used. It was shown that in general \textit{there is no free lunch for the privacy-utility trade-off} and one has to trade the preserving of privacy with a certain degree of degraded utility. The quantitative analysis illustrated in this article may serve as the guidance for the design of practical federated learning algorithms.  

\end{abstract}



\maketitle


\input{sections/01-introduction}

\input{sections/02-related}

\input{sections/03-The-Model}

\input{sections/04-framework}

\input{sections/05-discussion}

\input{sections/06-conclusion}

\textbf{ACKNOWLEDGEMENTS}

This work was partially supported by the National Key Research and Development Program of China under Grant 2018AAA0101100 and Hong Kong RGC TRS T41-603/20-R.

\bibliography{main}
\bibliographystyle{ACM-Reference-Format}

\clearpage

\appendix
\input{Appendix/AppendixA}
\input{Appendix/AppendixB}

\input{Appendix/AppendixC}

\input{Appendix/AppendixD}

\end{document}

%% file: commands.tex
\usepackage{afterpage}
\usepackage{prettyref}
\usepackage{framed}
\usepackage{booktabs}       
\newcommand{\pref}[1]{\prettyref{#1}}

\newcommand{\savehyperref}[2]{\texorpdfstring{\hyperref[#1]{#2}}{#2}}
\newrefformat{eq}{\savehyperref{#1}{Eq. \textup{(\ref*{#1})}}}
\newrefformat{eqn}{\savehyperref{#1}{Eq.~\textup{(\ref*{#1})}}}
\newrefformat{lem}{\savehyperref{#1}{Lemma~\ref*{#1}}}
\newrefformat{assump}{\savehyperref{#1}{Assumption~\ref*{#1}}}
\newrefformat{defi}{\savehyperref{#1}{Definition~\ref*{#1}}}
\newrefformat{lemma}{\savehyperref{#1}{Lemma~\ref*{#1}}}
\newrefformat{line}{\savehyperref{#1}{Line~\ref*{#1}}}
\newrefformat{thm}{\savehyperref{#1}{Theorem~\ref*{#1}}}
\newrefformat{corr}{\savehyperref{#1}{Corollary~\ref*{#1}}}
\newrefformat{cor}{\savehyperref{#1}{Corollary~\ref*{#1}}}
\newrefformat{sec}{\savehyperref{#1}{Section~\ref*{#1}}}
\newrefformat{app}{\savehyperref{#1}{Appendix~\ref*{#1}}}
\newrefformat{ex}{\savehyperref{#1}{Example~\ref*{#1}}}
\newrefformat{fig}{\savehyperref{#1}{Figure~\ref*{#1}}}
\newrefformat{alg}{\savehyperref{#1}{Algorithm~\ref*{#1}}}
\newrefformat{rem}{\savehyperref{#1}{Remark~\ref*{#1}}}
\newrefformat{conj}{\savehyperref{#1}{Conjecture~\ref*{#1}}}
\newrefformat{prop}{\savehyperref{#1}{Proposition~\ref*{#1}}}
\newrefformat{proto}{\savehyperref{#1}{Protocol~\ref*{#1}}}
\newrefformat{prob}{\savehyperref{#1}{Problem~\ref*{#1}}}
\newrefformat{claim}{\savehyperref{#1}{Claim~\ref*{#1}}}
\newrefformat{que}{\savehyperref{#1}{Question~\ref*{#1}}}
\newrefformat{op}{\savehyperref{#1}{Open Problem~\ref*{#1}}}
\newrefformat{fn}{\savehyperref{#1}{Footnote~\ref*{#1}}}
\newrefformat{property}{\savehyperref{#1}{Property~\ref*{#1}}}

\usepackage{xspace}
\usepackage{amsmath}
\usepackage{amsthm}
\usepackage{bbm}
\usepackage{algorithm}
\usepackage[algo2e, vlined, ruled]{algorithm2e}
\usepackage{algorithmic}
\usepackage{mathrsfs}
\usepackage{bm}
\usepackage{makecell}
\usepackage{tabulary}

\DeclareMathOperator*{\argmax}{argmax} 

\newcommand{\calD}{\mathcal{S}}
\newcommand{\calN}{\mathcal{N}}

\newcommand{\calA}{\mathcal{A}}
\newcommand{\calB}{\mathcal{B}}

\newcommand{\calS}{\mathcal{D}}
\newcommand{\calM}{\mathcal{M}}

\newcommand{\calW}{\mathcal{W}}

\newcommand{\calO}{\mathcal{B}}
\newcommand{\calRO}{\mathcal{O}}

\newcommand{\UU}{{\mathcal{U}_k}}

\newcommand{\calP}{\mathcal{P}}

\newcommand{\one}{\mathbbm{1}}

\newtheorem*{theorem*}{Theorem}

\newtheorem{assumption}{Assumption}[section]

\newtheorem*{lemma*}{Lemma}
\newtheorem{definition}{Definition}[section]

\theoremstyle{definition}
\newtheorem{property}{Property}

\newtheorem{thm}{Theorem}[section]
\newtheorem{prop}[thm]{Proposition}
\newtheorem{lem}[thm]{Lemma}

%% file: sections/01-introduction.tex
\section{Introduction}
\label{sec:introduction}

{Modern machine learning techniques are data-hungry, and it is not uncommon to use trillion bytes of data in developing large pre-trained machine learning models, e.g., for natural language processing \cite{devlin2018bert,dong2019unified} or image processing \cite{karras2019style,chen2021pre}. For a large variety of machine learning applications, for instance, in social media or finance use cases, data are often distributed across multiple devices or institutions.} Collecting such data onto a central server for training will incur additional communication overhead, management and business compliance costs, privacy issues, and even regulatory and judicial issues such as General Data Protection Regulation (GDPR)\footnote{GDPR is applicable as of May 25th, 2018 in all European member states to harmonize data privacy laws across Europe. https://gdpr.eu/}. Federated Learning (FL) has been introduced as an effective technology to allow multiple parties to jointly train a model \textit{without gathering or exchanging private training data among different parties}. Moreover, it is also required that information exchanged during the learning or inference stages \textit{do not disclose private data owned by respective parties to semi-honest adversaries} (for a formal definition, see Sect. \ref{sec:related:attack}), who aim to espy or infer other parties' private data from exchanged information. The requirements as such constitute the primary mandate for a novel statistical framework of privacy-preserving federated learning as illustrated in this article. 

From an information theory point of view, the amount of information about private data that a semi-honest party can infer from exchanged information is inherently determined by the \textit{statistical dependency} between private data and publicly exchanged information. 
Recent studies have shown that semi-honest adversaries can exploit this dependency to recover the private training images with pixel-level accuracy from exchanged gradients of learned models, e.g., by using Bayesian inference attacks (see Sect. \ref{sec:Bayesian inference attack} and \cite{zhu2019dlg,zhu2020deep,he2019model} for details). 
Taking such threat models into account, we therefore reiterate that a 
\textit{secure federated learning} (SFL) scheme must a) support multiple parties to jointly train models and to make joint inferences \textit{without exchanging private data}; b) have \textit{clearly defined threat models and security guarantees} under these models.

In this article, we consider a \textit{horizontal federated learning} \cite{yang2019federated}, or a \textit{cross-device federated learning} \cite{DBLP:journals/ftml/KairouzMABBBBCC21} setting, in which multiple \textit{clients} upload respective local models or model updates to an \textit{aggregator}\footnote{We assume the existence of a centralized aggregator for the brevity of analysis. Yet the analysis is also applicable to the decentralized federated learning setting in which model aggregation is performed without an aggregator. Detailed discussion on this setting is out of the scope of this article though and will be reported elsewhere.}, who is responsible for 
aggregating multiple local models into a global model. There are a variety of application scenarios that use this scheme for federated learning \cite{mcmahan2016federated, mcmahan2017communication, yang2019federated_new}.
The fundamental privacy-preserving requirement is to maintain potential \textit{privacy loss} below an acceptable level. This is achieved by reducing the dependency between the exchanged model information and private data. The reduced dependency, however, makes the aggregated global model less accurate and leads to \textit{utility loss}, such as a loss in model accuracy, as compared to a global model trained without using any protection. Severely degraded model utility indeed defeats the purpose of federated learning in the first place. We therefore propose in this article a unified statistical framework to analyze the privacy-utility trade-off on a rigorous theoretical foundation. The main results of our research are summarized as follows (also see Fig. \ref{fig:framework} for a pictorial summary):

\begin{itemize}
\item First, we formulate the privacy leakage attack \cite{zhu2019dlg,zhu2020deep,he2019model} through a \textbf{Bayesian inference attack} perspective (Def. \ref{def:attackA}) \cite{geman1984stochastic,wasserman2004bayesian}. A privacy metric called \textbf{Bayesian Privacy} is then proposed to quantify the additional amount of information about private data that a semi-honest adversary might gain by observing publicly exchanged information. Specifically, the information gain for adversaries i.e., \textbf{Bayesian privacy leakage} (Def. \ref{defi: average_privacy_JSD}) is measured by distances between distributions of adversaries' \textit{prior} and \textit{posterior} beliefs about private data. 
Also, an upper bound of Bayesian privacy leakage called \textbf{$\epsilon$-Bayesian Privacy} provides the privacy-preserving guarantee regardless of any Bayesian inference attack that may be launched by semi-honest adversaries. 

\item Second, we put forth a statistical framework to cast the optimal privacy-utility trade-off as a constrained optimization problem
in which the \textit{utility loss} (Def. \ref{defi: utility_loss}) is minimized subject to a preset upper bound of Bayesian privacy leakage (i.e. \textbf{$\epsilon$-Bayesian Privacy}) constraint (see Sect. \ref{sec: optimzation problem}). A theoretical analysis of the trade-off is then manifested as \textbf{Theorem \ref{thm: utility-privacy trade-off_JSD_mt}}, which dictates that a) a weighted sum of the privacy loss and utility loss is greater than a problem-dependent non-zero constant; b) relative weights between the privacy loss and utility loss depend on multiple factors including Bayesian inference attacks adversaries may launch, 
protection mechanisms adopted as well as distributions of the training and testing data. 
That is to say, in principle,\textit{ one has to trade a decrease of the privacy-loss with a certain degree of increase of the utility loss, and vice versa}\footnote{This principle bears an intriguing similarity to the no free lunch theorem for optimization and machine learning \cite{wolpert1997no}, which inspire us to name Theorem \ref{thm: utility-privacy trade-off_JSD_mt} “No free lunch theorem (NFL) for security and utility” in federated learning.}. 

\item Third, the general principle of Theorem \ref{thm: utility-privacy trade-off_JSD_mt} is applied to give quantitative analysis of the trade-off under specific privacy-protection mechanisms, including \textit{Randomization} \cite{geyer2017differentially,truex2020ldp,abadi2016deep}, \textit{Sparsity} \cite{shokri2015privacy, gupta2018distributed, thapa2020splitfed} and \textit{Homomorphic Encryption (HE)} \cite{gentry2009fully,batchCryp} (see \textbf{Theorem \ref{thm: privacy-utility-tradeoff-random}, \ref{thm: privacy-utility-tradeoff-sparsity} and \ref{thm:tradeoff-HE}}). These theoretically justified results are of interest in their own rights. On the other hand, they are of pragmatic values in serving as a practical guidance to the design of federated learning algorithms under various scenarios (see Sect. \ref{sec:application} for applications). 

\end{itemize}

To our best knowledge, the proposed Bayesian Privacy formulation is the first unified framework that explicitly considers both \textit{threat models} (attacking mechanisms) and \textit{security models} (protection mechanisms). It is also the first that is applicable to different privacy-protection mechanisms including randomization \cite{geyer2017differentially,truex2020ldp,abadi2016deep}, Sparsity \cite{shokri2015privacy, gupta2018distributed, thapa2020splitfed} and Homomorphic Encryption \cite{gentry2009fully,batchCryp}. {The rest of the article is organized as follows: Sect. \ref{sec:related} reviews existing works related to privacy measurements and federated learning; Sect. \ref{sec:framework} illustrates the statistical framework based on Bayesian inference attacks and protection mechanisms; 
Sect. \ref{sect:tradeoff} illustrates the quantification of the privacy-utility trade-off (Theorem \ref{thm: utility-privacy trade-off_JSD_mt}), the applications of the trade-off on different protection mechanisms are further presented in Sect. \ref{sec:application}; The comparison of Bayesian Privacy with other privacy measurements is presented in Sect. \ref{sect:compare-BP}. Finally, Sect. \ref{sec:conclusion} concludes the article with discussions on future research plans.}

\begin{figure}[h]
\centering
\includegraphics[width = 0.9\columnwidth]{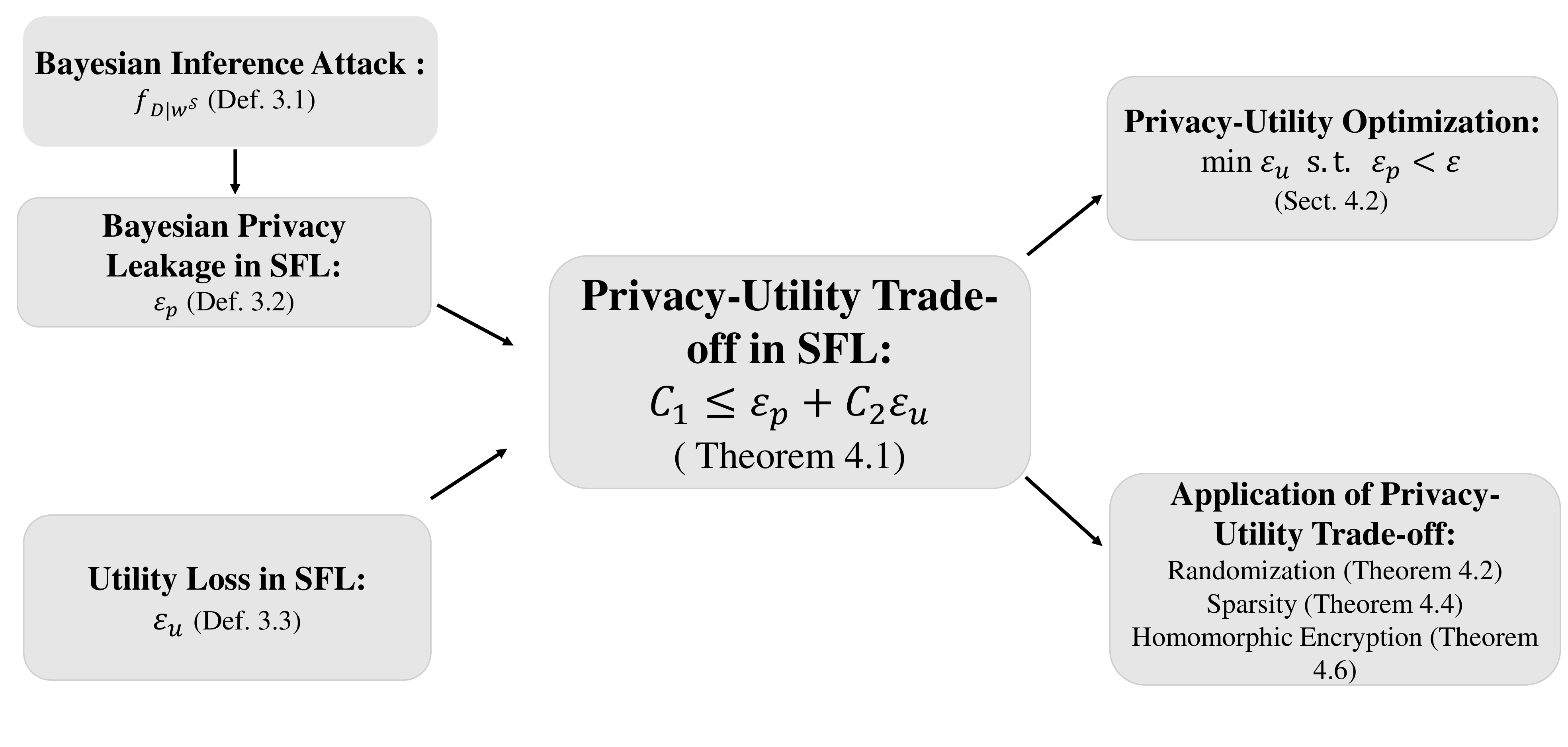}
\caption{Outline of the proposed SFL framework based on Bayesian inference attack and the privacy-utility trade-off.} \label{fig:framework}
\end{figure}

%% file: sections/02-related.tex
\section{Related Literature}
\label{sec:related}

\label{sec:related:DP}


This section gives a brief review of related work from distinct aspects, including \textit{privacy measurements, privacy attacking and protection methods in federated learning} and \textit{privacy-utility trade-off}. 

\subsection{Privacy Measurement}


The need of \textit{privacy-persevering computing} arises in various applications ranging from {database mining}, communication of secret information over public channels, and more recently, machine learning. Numerous definitions of privacy or privacy measurements proposed in the literature can be broadly categorized as follows as \textit{prior-independent} and \textit{prior-dependent} ones. 


\begin{itemize}
\item \textbf{Differential Privacy} (DP) is a celebrated \textit{prior-independent} definition proposed by Dwork et al. to protect individual privacy in response to queries about databases \cite{dwork2006calibrating, dwork2006differential,dwork2014algorithmic}. A series of DP variants were subsequently proposed to tighten the privacy budget under various conditions including Renyi Differential Privacy (RDP) with the natural relaxation of DP \cite{mironov2017renyi}, Gaussian Differential Privacy (GDP) with a closed-form privacy bound and stronger guarantee \cite{bu2020deep}.
Moreover, a large number of algorithms have been designed to achieve DP in federated learning (FL) \cite{seif2020wireless,truex2020ldp} and Abadi et al. proposed the moments accountants of DP suitable for deep learning \cite{abadi2016deep}. We refer readers to \cite{dwork2008differential,ha2019differential} for a thorough coverage of privacy-preserving based on differential privacy. 
It is worth mentioning that DP-based measurements do not take into account the prior distributions of private data, i.e., DP and its variants are \textit{prior-independent} except for Bayesian Differential Privacy (BDP) \cite{BDP-icml20}. 
 


\item Among different \textit{prior-dependent} privacy measurements \cite{du2012privacy,asoodeh2018estimation, issa2019operational}, \textbf{Information Privacy} (IP) \cite{du2012privacy} is proposed to model the \textit{privacy leakage} 
in terms of “amount of knowledge” learned by an adversary about the private data after observing the user’s output. The optimal privacy-utility trade-off is then cast, following the \textit{rate-distortion theory} \cite{Tretiak1974RateDT,blau2019rethinking}, as an optimization problem to minimize information leakage subject to utility loss constraints. The divergences adopted to quantify information leakages include KL divergence (mutual information) \cite{du2012privacy}, chi-squared divergence \cite{hsu2018generalizing}, total variation divergence \cite{rassouli2019optimal} and Renyi divergence \cite{liao2019tunable}.

\end{itemize}

\subsection{Federated Learning}
The notion of Federated Learning (FL) was initially proposed by McMahan et al. with the aim to build a machine learning model based on datasets that are distributed across multiple devices \cite{mcmahan2016federated,mcmahan2017communication, konevcny2016federated, konevcny2016federated_new}. The main idea is to aggregate local models learned on multiple devices, yet, without sending private data to a \textit{semi-honest} third-party server or other devices. Bearing in mind the privacy concern of secret data distributed across multiple institutions, Yang et al. extended applications of FL to a wide spectrum of use cases \cite{yang2019federated} and classified FL scenarios into three categories:
\begin{enumerate}
\item \textit{horizontal federated learning}: datasets share the same feature space but different space in samples; 

\item \textit{vertical federated learning}: two datasets share the same sample ID space but differ in feature space; 

\item \textit{federated transfer learning}: two datasets differ not only in samples but also in feature space.
\end{enumerate}
Yang et al. \cite{yang2019federated} also proposed the notion of \textit{secure federated learning} to highlight the importance of protecting the privacy of training data from various kinds of adversary attacks. Horizontal and vertical federated learning are also referred to as cross-device and cross-silo federated learning, respectively, in  \cite{DBLP:journals/ftml/KairouzMABBBBCC21}. A wealth of literature as follows has been proposed to improve the privacy-preserving capability and model utility.

\subsubsection{Threat model in Federated Learning}\label{sec:related:attack}

For the sake of privacy security, it is often necessary to consider the existence of \textit{semi-honest (honest-but-curious)} adversaries in FL. The adversary is honest in the sense that he/she faithfully follows the collaborative
learning protocol and does not submit any malformed message, but he/she may launch privacy attacks to espy the training data of other participants, by analyzing periodic updates to the joint model (e.g., gradients) during training. Such kind of attacks is referred to as \textit{Bayesian inference attack} (please refer to \pref{def:attackA}), which can be broadly classified according to the information source exploited by a semi-honest adversary:

\begin{itemize}
\item \textbf{Gradient inversion attack:} Zhu et al. \cite{zhu2019dlg,zhu2020deep} demonstrated that \textit{deep leakage} attacks allow adversaries to restore the private data up to pixel-level accuracy from exchanged deep neural network \textit{model gradients}. Following this seminal work, it was shown that adversaries can launch even more effective attacks by further exploiting various kinds of prior such as image smoothness prior by total variation loss \cite{geiping2020inverting}, 
image label prior \cite{zhao2020idlg}, and group consistency of estimated images \cite{yin2021see}. 

\item \textbf{Model Inversion attack from outputs:} \textit{Split learning} \cite{gupta2018distributed} 
allows a machine learning model to be separated and trained on multiple clients with low computing resources where each client only trains a small portion of the split models (e.g., a few layers of neural networks). However, adversaries could still infer private data from the \textit{model output} in this case \cite{fredrikson2015model,he2019model}. Gu et al. \cite{gu2021federated} later applied model inversion attack in a federated split learning (called \textit{splitfed}) setting \cite{thapa2020splitfed} in which clients only upload partial information to the server.

\item \textbf{GAN-based attack:} 
Generative Adversarial Network (GAN) networks were used to infer clients' private data. Hitaj et al. \cite{hitaj2017deep} viewed the aggregate model as a discriminator of GAN to generate a distribution of specific classes. 
Wang et al. \cite{wang2019beyond} proposed to learn a generator, which could recover user-specified private data in addition to the data distribution of a specific class.
\end{itemize} 

\subsubsection{Protection Mechanisms in Federated Learning}
\label{sec:related:prot}
In order to protect private data from being disclosed by adversarial attacks, the following protection mechanisms have been adopted for secure federated learning\footnote{Note that the privacy analysis illustrated in Sect. \ref{sec:application} applies to \textit{Randomization}, \textit{Sparsity} and \textit{HE} protection mechanisms and \textit{secret sharing}.}:


\begin{itemize}
\item {\textbf{Randomization}:}
\textit{Differential privacy} has been widely adopted to protect exchanged information in FL \cite{geyer2017differentially,truex2020ldp,abadi2016deep} by adding to model information either Laplace noise or Gaussian noise \cite{dwork2006calibrating}. \textit{Local differential privacy} \cite{dwork2006differential} was also proposed to randomize response in federated learning \cite{truex2020ldp,seif2020wireless, zhao2020local}. Despite its simplicity and popularity, the \textit{randomization} approach inevitably leads to compromised performances in terms of slow convergence, low model utility, and loose privacy guarantee as documented in \cite{BDP-icml20,kim2021federated,DBLP:journals/pvldb/HuYYDCYGZ15} etc.

\item \textbf{Encryption} is a widely-adopted technique for protecting sensitive information.
In particular, \textbf{Homomorphic Encryption} (HE) and its variants \cite{gentry2009fully,rivest1978data,gentry2010toward} 
allows the aggregation of FL models to be
performed directly on encrypted local models without decryption needed \cite{zhang2020batchcrypt,zhang2019pefl, aono2017privacy, truex2019hybrid}. 
However, extremely heavy computation and communication overhead incurred by HE prevent it from being readily applicable to large models such as 
deep neural networks with billions of model parameters. It remains an active research topic to come up with efficient HE algorithms to improve FL efficiency as demonstrated in \cite{aono2017privacy,batchCryp}.

\item
\textbf{Sparsity} based methods protect clients' private data 
by hiding part of the information from being exchanged with other parties \cite{shokri2015privacy}.
In particular, \textit{Split Learning} (SL) in a federated setting proposed to separate and hide part of neural network models from other parties \cite{gupta2018distributed} and
splitfed \cite{thapa2020splitfed} combined SL and FL to improve the efficiency and utility at the same time.

\item \textbf{Secret Sharing} \cite{SecShare-Adi79,SecShare-Blakley79,bonawitz2017practical} were developed to distribute a secret among a group of participants. However, it requires extensive exchange of messages and entails a communication overhead not viable in many federated learning settings.

\end{itemize}
Some protection mechanism like \cite{li2020tiprdc} does not belong to the field of federated learning, and is beyond the scope of our article.

\subsection{Privacy-Utility Trade-Off}

In the past decade, there has been wide interest in literature in understanding the privacy-utility trade-off:

\begin{itemize}
\item Some work focus on privacy-utility trade-off, where utility is quantified (inversely) via distortion (accuracy), and privacy via \textit{Information Privacy}. Sankar et al. \cite{sankar2013utility} provided a precise quantification of the trade-off in databases between the privacy needs of the individuals and the utility of the published data. Makhdoumi et al. \cite{makhdoumi2013privacy} modeled the privacy-utility trade-off according to the framework proposed by \cite{du2012privacy}. They regard the trade-off as a \textit{convex optimization} problem, which aims at minimizing the mutual information between the private data and released data under the constraint of utility (distortion) of released data. Moreover, Rassouli et al. \cite{rassouli2019optimal} further defined the privacy using total variation distance and illustrated that the optimal privacy-utility trade-off could be solved using a standard linear program. Also, Wang et al. \cite{wang2017estimation} provided a trade-off when utility and privacy were both evaluated using $\chi^2$-based information measures.

\item 
The trade-off of utility with relaxed privacy-preserving capability was considered necessary by Dwork et. al. for queries to be useful \cite{Imposs-DworkNaor10} and she also gave an quantitative analysis that relates $(\epsilon,0)$-DP with the maximal error (in utility) bounded by $\Omega(\frac{1}{\epsilon})$ \cite{dwork2014algorithmic}. 
Differential privacy budget $\epsilon$ of privacy-preserving algorithms was later related to the convergence (utility) of privacy-preserving algorithms \cite{wang2017differentially,girgis2021shuffled}. 
\end{itemize}

To sum up, in existing work we witnessed a compelling need to lay down a solid foundation for the design of novel federated learning algorithms. Especially, the recent Bayesian inference type of threats \cite{zhu2019dlg,zhu2020deep,geiping2020inverting,zhao2020idlg,yin2021see} were not explicitly modeled by existing privacy definitions, e.g., Differential Privacy which is considered the gold standard definition of privacy. This shortage of rigorous formulations makes it unclear for practitioners how to determine the optimal choice of algorithm design concerning the privacy-utility trade-off. There are no existing frameworks that simultaneously take into account different protection mechanisms, including randomization and Homomorphic Encryption. The rest of the present article therefore will illustrate such a novel statistical framework based on Bayesian inference attack which aims to address aforementioned issues not properly resolved in existing work.

%% file: sections/03-The-Model.tex
\section{General Setup and Framework}\label{sec:framework}
\subsection{Notations}
We follow the convention of representing the random variables using uppercase letters such as $S$, representing the particular values they take on using lowercase letters, and representing the support of random variables using copperplate such as $s$ and $\calS$. The distributions are represented using uppercase letters such as $F$ and $P$, and the probability density functions are represented using lowercase letters such as $f$ and $p$. We use $f_{D_k}(d)$ to represent the value of the probability density function $f$ at $s$, and the subindex represents the random variable. We use the notation $f_{D_k|W_k}(d|w)$ to represent the conditional density function. For continuous distributions $P$ and $Q$ over $\mathbb{R}^n$, the Kullback-Leibler divergence is defined as $\text{KL}(P||Q) = \int p(x)\log (p(x)/q(x))dx$, where $p$ and $q$ denote the probability densities of $P$ and $Q$. The Jensen-Shannon divergence is a smoothed version of the Kullback-Leibler divergence, which is defined as $\text{JS}(P||Q)  = \frac{1}{2}\left[\text{KL}\left(P, M\right) + \text{KL}\left(Q, M\right)\right]$, where $M = (P + Q)/2$. The total variation distance between $P$ and $Q$ is $\text{TV}(P||Q) = \sup_{A\subset\mathbb R^n} |P(A) - Q(A)|$. See Table \ref{table: notation} for detailed descriptions of notations.

\begin{table*}[!htp]
\footnotesize
  \centering
  \setlength{\belowcaptionskip}{15pt}
  \caption{Table of Notation}
  \label{table: notation}
    \begin{tabular}{cc}
    \toprule
    Notation & Meaning\cr
    \midrule\
    $\epsilon_{p}$ & Privacy leakage (Def. \ref{defi: average_privacy_JSD})\cr
    $\epsilon_u$ & Utility loss (Def. \ref{defi: utility_loss}) \cr
    $W^{\calRO}_k$ & Unprotected model information of client $k$\cr
    $W^{\calD}_k$ & Protected model information of client $k$\cr
 $P^{\calRO}_k$ & Distribution of unprotected information of client $k$\cr
 $P^{\calD}_k$ & Distribution of protected information of client $k$\cr
 $F^{\calO}_k$ & Adversary's prior belief distribution about the private information of client $k$\cr
 $F^{\calA}_k$ & Adversary's belief distribution about client $k$ after observing the protected information\cr
 $F^{\calRO}_k$ & Adversary's belief distribution about client $k$ after observing the unprotected information\cr
 $\text{JS}(\cdot||\cdot)$ & Jensen-Shannon divergence between two distributions\cr
 $\text{TV}(\cdot||\cdot)$ & Total variation distance between two distributions\cr
    \bottomrule
    \end{tabular}
\end{table*}

\subsection{General Setup}

In this work, we consider a \textit{horizontal federated learning} scenario, or cross-device federated learning \cite{DBLP:journals/ftml/KairouzMABBBBCC21}, where multiple ($K$) participants\footnote{In this article we use \textit{participant, party, client} interchangeably to refer to a device or institute that participates in the federated learning. } collaboratively learn a global model without exposing their private training data \cite{mcmahan2017communication, yang2019federated}. The requirement of adopting certain protection of private data is due to the threat model. Some \textit{semi-honest} adversaries may launch \textit{privacy attacks} on exchanged information to infer private data of other participants (for a formal definition of semi-honest adversaries, see Sect. \ref{sec:related:attack}).
Initially, the server distributes the global model information to all the clients. The overall secure federated learning (SFL) procedures are illustrated in the left panel of Fig. \ref{fig: setup} and summarized as follows:
\begin{enumerate}[label=\circled{\arabic*}]
\item With the global model information from the server, each client $k$ trains the \textit{local model} using his/her own data set $D_k$, and obtains the \textit{local model information} $W_k^{\calRO}$, which follows a distribution $P_k^{\calRO}$. The objective of learning $W_k^{\calRO}$ is to maximize the model utility $U_k$.
\item In order to prevent semi-honest adversaries from inferring other clients' private information $D_k$ according to $W_k^{\calRO}$, the client adopts a protection mechanism $M$ to convert model information $W_k^{\calRO} $ to protected model information $W_k^{\calD}$. 
\item Each client uploads the protected information $W_k^{\calD}$ to the server and aggregates as a new global model $W_{a}^{\calD}$.
\item The clients download the global model $W_{a}^{\calD}$ and continue to update the local model information.
\end{enumerate}
The processes \textcircled{1}-\textcircled{4} iterate until the utility of the aggregated model does not improve.\\
\textbf{Remark:} In SFL, the local model information includes the \textit{model parameters}, \textit{model gradients} and \textit{model outputs}, all of which may optionally be exchanged to the aggregator and get exposed to semi-honest adversaries (see details in the Appendix \ref{App:bayes-inference-attack}). The goal of the protection mechanism $M$ is to protect private data such that the dependency between $W_k^{\calD}$ and $D_k$ is reduced, as compared to the dependency between the unprotected information $W_k^{\calRO}$ and $D_k$.

\begin{figure}[t]
\centering
\includegraphics[width = 0.9\columnwidth]{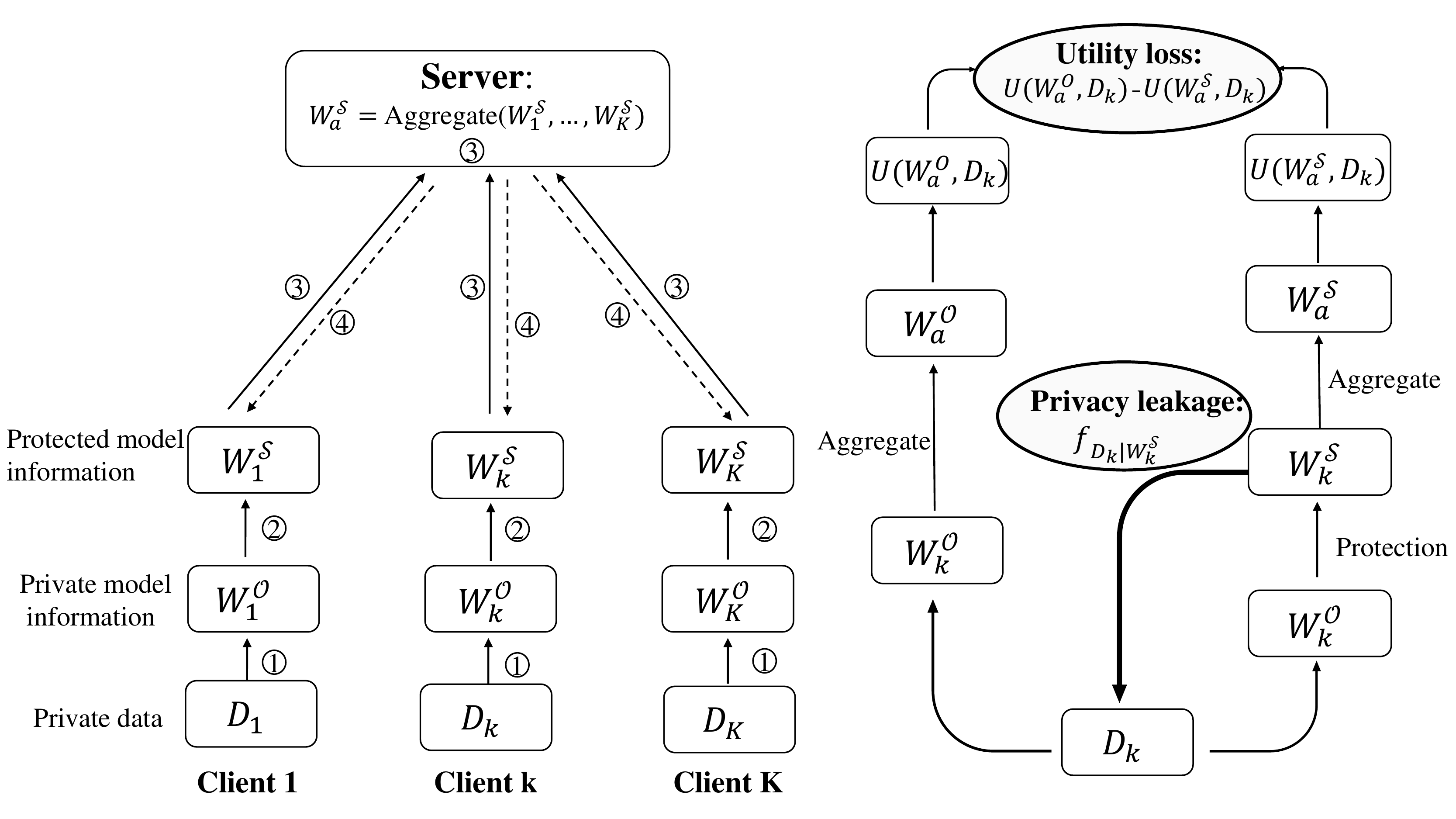}
\caption{An illustration of the setup for secure federated learning. The left panel demonstrates the four procedures: \textcircled{1} The $k_{th}$ \textit{client} learns the \textit{local model} information $W_k^{\calRO}$ using his/her own data set $D_k$; \textcircled{2} The client adopts a protection mechanism $M$ to convert model information $W_k^{\calRO} $ to protected model information $W_k^{\calD}$; \textcircled{3} Each client uploads the protected information $W_k^{\calD}$ to the \textit{aggregator} who aggregates all $W_k^{\calD}, k=1,2,\cdots,K$ into the global model $W_{a}^{\calD}$; \textcircled{4} The clients download the global model $W_{a}^{\calD}$ and continues to update the local model information. The right panel illustrates the privacy leakage that adversaries may infer the private data $D_k$ according to the protected model information $W_k^{\calD}$, and the utility loss that protection mechanism induces the aggregated global model to be less accurate with the decreased utility.
}
\label{fig: setup}
\vspace{-5pt}
\end{figure}

In this SFL setting, the privacy loss and utility loss as illustrated in right panel of Fig. \ref{fig: setup} are considered:
\begin{enumerate}
\item For the sake of privacy-preserving, a Bayesian privacy leakage measure (see Def. \ref{defi: average_privacy_JSD}) is used to quantify the amount of information about private data that semi-honest adversaries may still infer in spite of the protection mechanisms applied to publicly exposed information. The Bayesian privacy leakage allows one to evaluate the security of a secure federated learning scheme and justifies that the adopted protection mechanism is secure in thwarting Bayesian inference attacks if the Bayesian privacy leakage is less than an acceptable threshold (see Sect. \ref{sec: optimzation problem}).
\item The protection mechanism modifies the original model information $W_k^{\calRO}$ to its protected counterpart $W^{\calD}_k$, and induces the local model to behave less accurately. Consequently, the aggregated global model is less accurate and the incurred utility loss is defined as the difference of utilities with and without protections (Def. \ref{defi: utility_loss}).
\end{enumerate}



\subsection{Bayesian Inference Attack in Secure Federated Learning}
\label{sec:Bayesian inference attack}

It was shown that semi-honest adversaries could recover private training images up to pixel-level accuracy from unprotected gradients of learned models, e.g., by using efficient Bayesian inference attacks (\cite{zhu2020deep,he2019model}).
We assume that an adversary aims to recover the $k_{th}$ client's private variable $D_k$ from exposed variable $W_k^{\calD}$, which is the output of applying certain protection mechanisms on model information i.e., $W_k^{\calD} = M(W_k^{\calRO})$.
Such a Bayesian inference attack is formally defined as follows.
\begin{definition}[Bayesian inference attack] \label{def:attackA}
A \textbf{Bayesian inference attack}\footnote{The Bayesian inference framework has been applied to the image restoration problem since the 1960s~\cite{dempster1968generalization,imgRestorGeman84,box2011bayesian}.} 
is an optimization process that aims to infer the private variable $D_k$ in order to best fit the exposed information $W_k^{\calD}$ as: 
\begin{equation} \begin{split} \label{eq:bayes-infer-attack}
d^{*} &=\argmax\limits_{d}\log\left(f_{D_k|W_k^{\calD}}(d|w)\right) \\
& =\argmax\limits_{d}\log\left(\frac{f_{W_k^{\calD}|D_k}(w|d) f_{D_k}(d)}{f_{W_k^{\calD}}(w)}\right) \\
&=\argmax\limits_{d} [\log f_{W_k^{\calD}|D_k}(w|d)+\log f_{D_k}(d)],\\ 
\end{split}
\end{equation}
\end{definition}
where $f_{D_k|W_k^{\calD}}(d|w)$ is the posterior of $D_k$ given the protected variable $W_k^{\calD}$. According to Bayes' theorem, maximizing the log-posterior $f_{D_k|W_k^{\calD}}(d|w)$ on $D_k$ involves maximizing summation of $\log(f_{W_k^{\calD}|D_k}(w|d))$ and $\log(f_{D_k}(d))$. The former one aims to find $D_k$ to best match $W_k^{\calD}$ (maximize the likelihood of $W_k^{\calD}$) and the latter one aims to make the prior of $D_k$ more significant. In short, Bayesian inference attack establishes the posterior belief of $D_k$ conditioned on $W_k^{\calD}$, denoted as $f_{D_k|W_k^{\calD}}$.

The learned conditional distribution $f_{D_k|W_k^{\calD}}$ from the Bayesian inference attack reflects the dependency between $W_k^{\calD}$ and $D_k$, which determines the amount of information that adversaries may infer about $D_k$ after observing $W_k^{\calD}$. In case the exposed information $W_k^{\calD}$ is \textit{independent} of the private data $D_k$, then the posterior belief $f_{D_k|W_k^{\calD}}$ is guaranteed to be \textit{indistinguishable} from the prior $f_{D_k}$ and the Bayesian privacy leakage is zero. This extreme case corresponds to 
the \textit{Semantic Security (SS)} in cryptosystems \cite{goldwasser1984probabilistic}.

\textbf{Remark:} \\
(1) Bayesian inference attack defined here is a concrete realization of the semi-honest participant assumption. The attacker learns information from other parties and attempts to infer the private information held by other parties. The attacker does not actively inject information on exchanged messages in order to gain more private information or ``poison'' the joint model, however.

(2) It is required that the time cost of the Bayesian inference attack is polynomial. Given this requirement, adversaries cannot gain any additional information about the private data of clients if the attacking cost increases exponentially, e.g., against Homomorphic Encryption \cite{gentry2009fully}. From a practical point of view, we give in Appendix \ref{App:time-complex} detailed analysis of Bayesian inference attack including \textit{gradient inversion} attack, \textit{model inversion} attack and \textit{brute-force} attack against different protection mechanisms.

(3) We formulate three attacks including gradient-inverse attack \cite{zhu2020deep,geiping2020inverting,zhao2020idlg, yin2021see}, model inversion attack \cite{fredrikson2015model, he2019model} and brute-force attack (especially for encryption) from the Bayesian inference perspective (see details in Appendix \ref{App:bayes-inference-attack}). These attacking approaches differ in the types of exposed information, including model, model gradients and model outputs in SFL. Moreover, the prior (denoted as $F^{\calO}_k$) is an important factor for inferring the private data. For example, if the attacker knows the label of an image is ``cat'', he/she may use an averaged ``cat'' image to initialize the restored image and significantly improve the accuracy of restored data (the influence of prior for the attack is shown in Appendix \ref{App:influ-prior}). 

\subsection{Bayesian Privacy Leakage}

With the Bayesian inference attack (introduced in Def. \ref{def:attackA}), we now define \textit{Bayesian privacy leakage} as a privacy leakage measurement, which quantifies the discrepancy between the posterior belief (with exposure $W_k^{\calD}$) and prior belief (without exposure) under the Bayesian inference attack. Let $F^{\calA}_k$, $F^{\mathcal O}_k$ and $F^{\calO}_k$ represent the attacker's belief distribution about $D_k$ upon observing the protected information, the original information and without observing any information respectively. Let $f^{\calA}_{D_k}$, $f^{\calRO}_{D_k}$ and $f^{\calO}_{D_k}$ represent the probability density function of $F^{\calA}_k$, $F^{\calRO}_k$ and $F^{\calO}_k$. Specifically, $f^{\calA}_{D_k}(d) = \int_{\mathcal{W}_k} f_{{D_k}|{W_k}}(d|w)dP^{\calD}_{k}(w)$, $f^{\calRO}_{D_k}(d) = \int_{\mathcal{W}_k} f_{{D_k}|{W_k}}(d|w)dP^{\calRO}_{k}(w)$, and $f^{\calO}_{D_k}(d) = f_{D_k}(d)$.
\begin{definition}[Bayesian privacy leakage]\label{defi: average_privacy_JSD}
Let $\epsilon_{p,k}$ represent the privacy leakage of client $k$, which is defined as
\begin{align}\label{eq: def_of_pl}
\epsilon_{p,k} = \sqrt{{\text{JS}}(F^{\calA}_k || F^{\calO}_k)} = \left[\frac{1}{2}\int_{\mathcal{D}_k} f^{\calA}_{D_k}(d)\log\frac{f^{\calA}_{D_k}(d)}{f^{\calM}_{D_k}(d)}\textbf{d}\mu(d) + \frac{1}{2}\int_{\mathcal{D}_k} f^{\calO}_{D_k}(d)\log\frac{f^{\calO}_{D_k}(d)}{f^{\calM}_{D_k}(d)}\textbf{d}\mu(d)\right]^{\frac{1}{2}},
\end{align}
where $f_{D_k}^{\calM}(d) = \frac{1}{2}(f^{\calA}_{D_k}(d) + f^{\calO}_{D_k}(d))$. Furthermore, the Bayesian privacy leakage in SFL resulted from releasing the protected model information is defined as
\begin{align}
\epsilon_p = \frac{1}{K}\sum_{k=1}^K \epsilon_{p,k}.
\end{align} 
\end{definition}

The Bayesian privacy leakage measures the discrepancy between the adversaries' belief with and without leaked information. Moreover, the Bayesian privacy leakage is averaged with respect to the protected model information variable which is exposed to adversaries.\\

\textbf{Remark:} Unlike KL divergence, {\text{JS}} divergence is symmetrical and its square root satisfies the triangle inequality\footnote{We employ the property of triangular inequality to establish main results illustrated in Sect. \ref{subsect:PU-tradeoff}. Also see Fig. \ref{fig: flowchart} for a pictorial illustration.} \cite{endres2003new}. This property allows the derivation of main results illustrated in \pref{thm: utility-privacy trade-off_JSD_mt}. The definition of Bayesian privacy leakage could be generalized as the maximum privacy leakage over clients.\\

\subsection{Utility Loss in Secure Federated Learning}

In addition to privacy leakage, another important concern investigated in our work is the utility loss, which is defined as follows:

\begin{definition}[Utility Loss]\label{defi: utility_loss}
The utility loss is defined as the discrepancy between the utility with the unprotected model information drawn from distribution $P_a^{\calRO}$ and that drawn from the protected distribution $P_a^{\calD}$,
\begin{align*}
\epsilon_{u} = \frac{1}{K}\sum_{k=1}^K \epsilon_{u,k} = \frac{1}{K}\sum_{k=1}^K [U_k(P_a^{\calRO}) - U_k(P_a^{\calD})],
\end{align*}
where $P_a^{\calRO}$ and $P_a^{\calD}$ represent, respectively, the distributions of \textit{convergent models} with or without any protection mechanism used, and $U_k(P) = \mathbb E_{D_k}\mathbb E_{W_k}\frac{1}{|D_k|}\sum_{d\in D_k}U(W_k,d)$ is the expected utility taken with respect to $D_k\sim P_k$ and $W_k\sim P$. 
\end{definition}
\textbf{Remark:} The model utility evaluates model performance for a variety of learning tasks. For example, model utility is classification accuracy in a classification problem and prediction accuracy in a regression problem.

%% file: sections/04-framework.tex
\section{A general framework for the Privacy-Utility Trade-Off}\label{sect:tradeoff}
In Sect. \ref{sec:framework}, we have introduced the definition of privacy leakage and utility loss in SFL. In this section, we first lay down the rigorous analysis of the trade-off between Bayesian privacy leakage and utility loss (no free lunch theorem). Then we formulate the \textit{privacy-utility trade-off} as an optimization problem, which minimizes the utility loss under the $\epsilon$-Bayesian privacy constraint. Finally, we apply the proposed no free lunch theorem between Bayesian privacy leakage and utility loss into four protection mechanisms in SFL.

\subsection{Theoretical Analysis for Privacy-Utility Trade-Off}\label{subsect:PU-tradeoff}

The idea of privacy-preserving is to modify the unprotected information and guarantee that the private data cannot be disclosed to semi-honest adversaries. In order to protect privacy, the protected information denoted as $W^{\calD}_k$ is shared. 
The utility loss quantifies the increase of expected average loss incurred by the released information $W^{\calD}_k$ compared with the utility caused by $W^{\calRO}_k$. 
Note that $W^{\calD}_k$ and $W^{\calRO}_k$ respectively follow distributions $P^{\calD}_k$ and $P_k^{\calRO}$. In this setting, we want to characterize the limitation of privacy-preserving mechanisms. It requires to measure the amount of utility that is inevitable to lose in order to protect privacy. Intuitively, the larger the change exerted on $W^{\calRO}_k$ is, the more secure the privacy protection is, and meanwhile, the less accurate the output is. We first measure the extent of modification using the total variation distance between the unprotected distribution and the protected distribution. Then, we use this distance as a key element for connecting the privacy leakage and the utility loss.
Note that the union of $\mathcal W_k^{\calD}$ (the support of $P_k^{\calD}$) and $\mathcal W_k^{\calRO}$ (the support of $P_k^{\calRO}$) is denoted as $\calW_k$.



\begin{definition}[Optimal parameters]
Let $\mathcal W^{*}_a$ represent the set of parameters achieving the maximum utility. Specifically, 
\begin{align*}
    \mathcal W^{*}_a = \argmax_{w\in\mathcal W_a}\frac{1}{K}\sum_{k=1}^K U_k(w),
\end{align*}
where $U_k(w) = \mathbb{E}_{D_k}\frac{1}{|D_k|}\sum_{d\in D_k}U(w,d)$ is the expected utility taken over $D_k$ sampled from distribution $P_k$.
\end{definition}





\begin{definition}[Near-optimal parameters]\label{defi: neighbor_set}
Let $\calW^{\calD}_a$ represent the support of the protected distribution of the aggregated model information. Given a non-negative constant $c$, the \textit{near-optimal parameters} is defined as
$$\calW_{c} = \left\{w\in\calW^{\calD}_a: \left|\frac{1}{K}\sum_{k=1}^K  U_k(w^{*})-\frac{1}{K}\sum_{k=1}^K U_k(w)\right|\le c, \forall w^{*}\in\mathcal W^{*}_a\right\}.$$


\end{definition}

\begin{assumption}\label{assump: assump_of_Delta}
Let $\Delta$ be the \textit{maximum} constant that satisfies


\begin{align}
     \int_{\mathcal W^{\calD}_a} p^{\calD}_{W_a}(w)\one\{w\in\calW_{\Delta}\} dw\le\frac{{\text{TV}}(P_a^{\calRO} || P_a^{\calD} )}{2},
\end{align}

where $p^{\calD}_{W_a}$ represents the probability density function of the protected model information $W_a^\calD$. We assume that $\Delta$ is positive, i.e., $\Delta >0$.
\end{assumption}
\textbf{Remark:}\\
(1) This assumption implies that the cumulative density of the near-optimal parameters as defined in Def. \ref{defi: neighbor_set} is bounded. This assumption excludes the cases where the utility function is constant or indistinguishable between the optimal parameters and a certain fraction of parameters.  \\
(2) Note that $\Delta_k$ is independent of the threat model of the adversary and $\Delta_k$ is a constant when the protection mechanism, the utility function, and the data sets are fixed.\\

The following theorem represents the central result of this work. As a form of the \textit{No Free Lunch (NFL)} theorem, it states that privacy-protection and utility enhancement of a joint model in our semi-honest SFL setting is bounded by a constant. If one increases privacy protection, one risks losing some utility of the joint model, and vice versa. 

\begin{thm}[No free lunch theorem (NFL) for security and utility]\label{thm: utility-privacy trade-off_JSD_mt} Let $\epsilon_p$ be defined in Def. \ref{defi: average_privacy_JSD}, we have that
\begin{align}\label{eq: total_variation-privacy trade-off}
C_1\le\epsilon_{p} + \frac{1}{K}\sum_{k=1}^K \frac{1}{2}(e^{2\xi}-1)\cdot {\text{TV}}(P_k^{\calRO} || P^{\calD}_k).
\end{align}

Furthermore, let $\epsilon_u$ be defined in Def. \ref{defi: utility_loss} at the convergence step, with \pref{assump: assump_of_Delta} we have that
\begin{align}\label{eq: total_variation-privacy trade-off_mt}
 C_1 \le\epsilon_{p} + C_2\cdot \epsilon_{u},
\end{align}
in which
\begin{itemize}
\item $\xi = \max_{k\in [K]} \xi_k$, where $\xi_k = \max_{w\in \mathcal{W}_k, d \in \mathcal{D}_k} \left|\log\left(\frac{f_{S_k|W_k}(d|w)}{f_{D_k}(d)}\right)\right|$ represents the maximum privacy leakage over all possible information $w$ released by client $k$, and $[K] = \{1,2,\cdots, K\}$. $\xi$ is a constant independent of the protection mechanism;

\item $C_1 = \frac{1}{K}\sum_{k=1}^K \sqrt{{\text{JS}}(F^{\calRO}_k || F^{\calO}_k)}$ is a constant 
representing the averaged square root of JS divergence between adversary's belief distribution about the private information of client $k$ before and after observing the unprotected parameter. This constant is independent of the protection mechanisms.

\item $C_2 = \frac{\gamma}{4\Delta}(e^{2\xi}-1)$ is a constant once the protection mechanisms, the utility function, and the data sets are fixed, where $\gamma = \frac{\sum_{k=1}^K {\text{TV}}(P_k^{\calRO} || P^{\calD}_k)}{{\text{TV}}(P^{\calRO}_a || P^{\calD}_a )}$\footnote{see details of analyzing of the value of $\gamma$ in Appendix \ref{sec:proof-random-app}}. 
\end{itemize}
\end{thm}

\pref{thm: utility-privacy trade-off_JSD_mt} illustrates that the summation of the average privacy leakage of the clients and the utility loss is lower bounded by a problem-dependent constant. As shown in Fig. \ref{fig: flowchart}, Eq. \eqref{eq: total_variation-privacy trade-off_mt} is based on the triangle inequality of the Jensen-Shannon measure adopted in Bayesian privacy leakage \ref{thm: utility-privacy trade-off_JSD_mt}. It essentially dedicates that \textit{one has to trade a decrease of the privacy leakage ($\epsilon_{p}$) with a certain degree of the increase of the utility loss ($\epsilon_u$) and vice versa}. This principle bears an intriguing similarity to the no free lunch 
theorem for optimization and machine learning \cite{wolpert1997no}, which dedicated that ``if an algorithm performs well on a certain class of problems then it necessarily pays for that with degraded performance on the set of all remaining problems''. We therefore name Eq. \eqref{eq: total_variation-privacy trade-off_mt} ``no free lunch theorem (NFL) for security and utility'' in federated learning.

\textbf{Remark:} ${\text{TV}}(P_k^{\calRO} || P^{\calD}_k)$ in Eq. \eqref{eq: total_variation-privacy trade-off} provides an upper bound of the modification brought by the protection mechanism, in terms of the distance between distributions of model information with and without protection methods being applied.


\begin{figure}[t]
\centering
\includegraphics[width = 0.8\columnwidth]{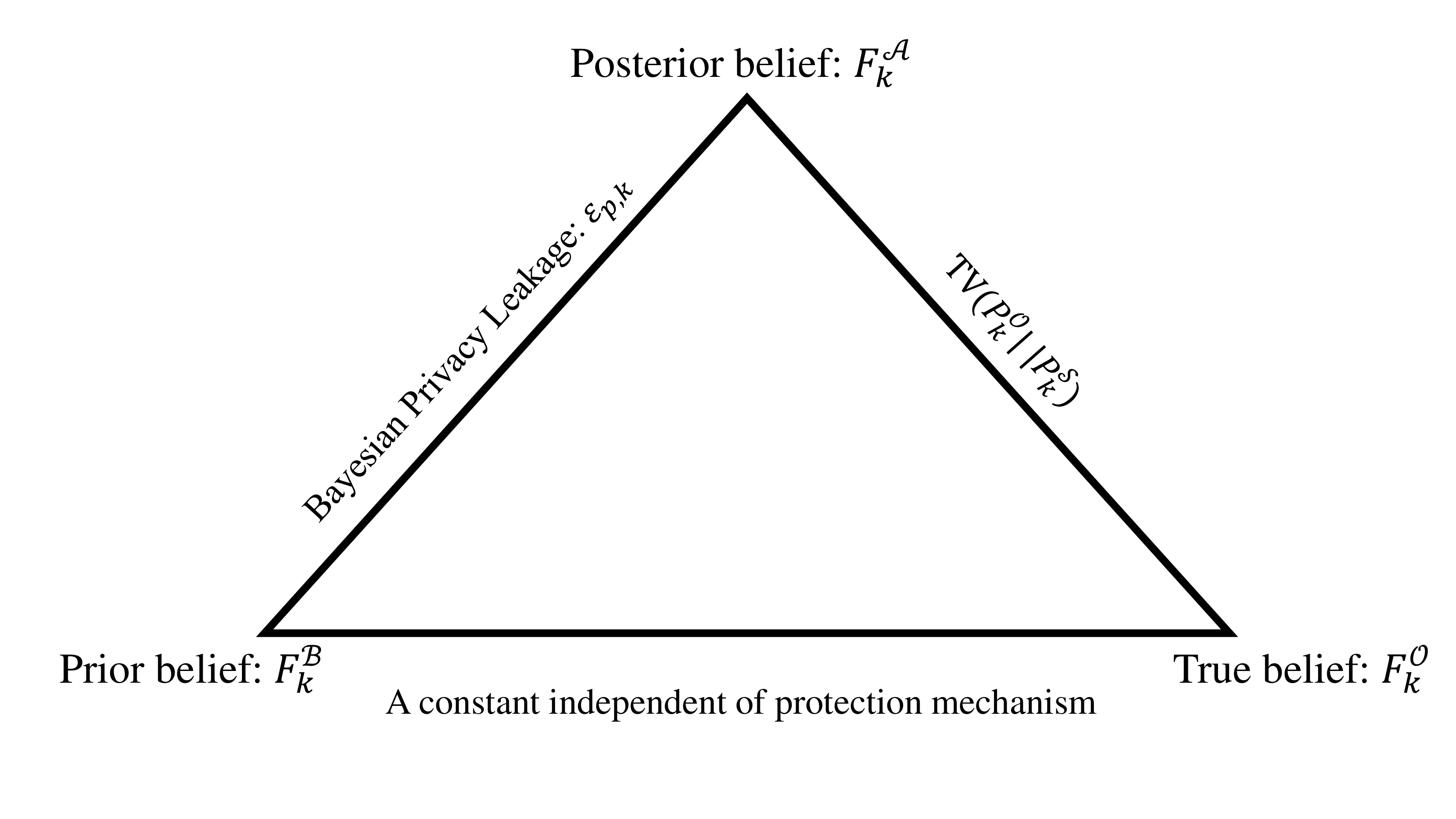}
\vspace{-5mm}
\caption{$\epsilon_{p, k} = \sqrt{{\text{JS}}(F^{\calA}_k || F^{\calO}_k)}$ as is introduced in Def. \ref{defi: average_privacy_JSD}. Notice that the square root of the Jensen-Shannon divergence satisfies the triangle inequality. It implies that the summation of $\epsilon_{p,k}$ and $\sqrt{{\text{JS}}(F^{\calA}_k || F^{\calRO}_k)}$ is at least $\sqrt{{\text{JS}}(F^{\calRO}_k || F^{\calO}_k)}$, which is a problem-dependent constant. Moreover, $\sqrt{{\text{JS}}(F^{\calA}_k || F^{\calRO}_k)}$ is at most $\frac{1}{2}(e^{2\xi}-1)\cdot\text{TV}(P_k^{\calRO} || P^{\calD}_k)$. Consequently, the summation of $\epsilon_{p,k}$ and $\frac{1}{2}(e^{2\xi}-1)\cdot\text{TV}(P_k^{\calRO} || P^{\calD}_k)$ is at least $\sqrt{{\text{JS}}(F^{\calRO}_k || F^{\calO}_k)}$. The total variation distance between the unprotected distribution and protected distribution is then used as a key element for connecting the privacy leakage and the utility loss.}
\label{fig: flowchart}
\end{figure}


\subsection{Privacy-Preserving Optimization} \label{sec: optimzation problem}
In the privacy-preserving machine learning scenario, clients aim to maintain maximal possible model utility without disclosing private information beyond an acceptable level. 
We can view this trade-off as an optimization problem. The goal is to minimize the utility loss subject to the privacy-preserving constraint. Let the Bayesian privacy leakage $\epsilon_{p}$ be defined as Def. \ref{defi: average_privacy_JSD}, 
we say a SFL system guarantees $\mathbf{\epsilon}$-\textbf{Bayesian Privacy} if it holds that 
\begin{equation}\label{eq:ep-bayes-privacy}
\epsilon_{p} \leq \epsilon.
\end{equation}
${\epsilon}$-{Bayesian Privacy} guarantees that the Bayesian privacy leakage of the SFL system never exceeds a prescribed threshold $\epsilon$ as the security requirement. For a federated machine learning system, the goal is to find a protection mechanism $M$ that achieves the minimum aggregated utility loss. These two requirements can be cast as a constrained optimization problem over the protection mechanism:

\begin{align} \label{eq: key_constraint_upper_bound}
\begin{array}{r@{\quad}l@{}l@{\quad}l}
\quad\min\limits_{M}&\epsilon_{u}: = \frac{1}{K}\sum_{k=1}^K [U_k(P_a^{\calRO}) - U_k(P_a^{\calD})],\\
\text{subject to} & \epsilon_p = \frac{1}{K}\sum_{k = 1}^K \sqrt{{\text{JS}}(F^{\calA}_k || F^{\calO}_k)}\le\epsilon.\\
\end{array}
\end{align}

\textbf{Remark:}

(1) The constant $\epsilon$ imposed in Eq. \eqref{eq: key_constraint_upper_bound} provides a privacy security guarantee that one can achieve, regardless of a variety of Bayesian inference attacks that may be launched by adversaries. 

(2) This constrained optimization problem is closely related to the optimization formulated \cite{du2012privacy} which instead aims to minimize the information leakage while the utility loss is guaranteed to be less than a given threshold.

\subsection{Applications of Privacy-Utility Trade-off} 
\label{sec:application}

In this section, we apply the no free lunch theorem between the Bayesian privacy leakage and utility loss on four privacy-preserving mechanisms: Randomization \cite{geyer2017differentially,truex2020ldp,abadi2016deep}, Sparsity \cite{shokri2015privacy, gupta2018distributed, thapa2020splitfed}, Homomorphic Encryption \cite{gentry2009fully,batchCryp} and Secret Sharing \cite{SecShare-Adi79,SecShare-Blakley79,bonawitz2017practical}\footnote{secure hardware\cite {liu2022distributed} was used in federated learning but we don't include it in our analysis since it requires to send private data to a trustworthy third-part which violates the semi-honest adversaries setting considered in this article.}. We analyze the influence of specific protection mechanisms' parameters on the fluctuation of privacy leakage and utility loss.

\subsubsection{Randomization Mechanism}
\hfill

Randomization mechanism is a natural choice since its noise parameter e.g., $\sigma$ allows a flexible control of the trade-off. In SFL, one of the \textit{Randomization} methods is to add random noise such as Gaussian noise to model gradients \cite{abadi2016deep,geyer2017differentially,truex2020ldp}. Let $W_k^{\calRO} \in \mathcal W^{\calRO}_k$ be the parameter sampled from distribution $P_k^{\calRO} = \calN(\mu_0,\Sigma_0), k =1, \cdots,K$, where $\mu_0 \in \mathbb{R}^n$, $\Sigma_0 = \text{diag}(\sigma_{1}^2,\cdots, \sigma_{n}^2)$ is a diagonal matrix. Adding noise in the parameter could be expressed as $W_k^{\calD} = W_k^{\calRO} + \epsilon_k$, where $\epsilon_k \sim \calN(0, \Sigma_\epsilon)$ and $\Sigma_\epsilon = \text{diag}(\sigma_\epsilon^2, \cdots, \sigma_\epsilon^2)$. Therefore, $W_k^{\calD}$ follows the distribution $P_k^{\calD} = \calN(\mu_0, \Sigma_0+ \Sigma_\epsilon)$. The server aggregates all the protected parameters as $W_a^{\calD} = \frac{1}{K}\sum_{k=1}^K (W_k^{\calRO} + \epsilon_k)$ following the distribution $P^{\calD}_a = \calN(\mu_0, \Sigma_0/K+ \Sigma_\epsilon/K)$. The following theorem establishes bounds for the privacy leakage and utility loss using the variance of the noise $\sigma_\epsilon^2$ (see Appendix \ref{sec:proof-random-app} for the full proof).




\begin{thm}\label{thm: privacy-utility-tradeoff-random} For the randomization mechanism by adding Gaussian noise, the privacy leakage and utility loss are bounded using variance of the Gaussian noise $\sigma_\epsilon^2$,

\begin{align} \label{tradeoff-random}
C_1
\leq \epsilon_{p} + \frac{C_3}{K}\min\left\{1,\sigma_\epsilon^2\sqrt{\sum_{i=1}^n\frac{1}{\sigma_{i}^4}}\right\},
\end{align}
and
\begin{align} \label{tradeoff-random-utility}
\epsilon_u \leq C_4 \min\left\{1,\sigma_\epsilon^2\sqrt{\sum_{i=1}^n\frac{1}{\sigma_{i}^4}}\right\},
\end{align}
where $C_1 = \frac{1}{K}\sum_{k=1}^K \sqrt{{\text{JS}}(F^{\calRO}_k || F^{\calO}_k)}, C_3 =(e^{2\xi}-1)/2$ are two constants independent of the protection mechanisms adopted, and $C_4$ is a constant satisfying that  $U(w,d)\leq C_4$ for any $w \in \mathcal{W}_k$ and $d \in \calS_k$.
\end{thm}

\textbf{Remark:} In this application, we assume $W_k^\calRO$ follows Gaussian distribution $\calN(\mu_0,\Sigma_0)$. In particular, when $\Sigma_0$ tends to zero, the Gaussian distribution degenerates into the one-point distribution when the algorithm converges. Moreover, the estimation of the variable $\gamma$ in \pref{thm: utility-privacy trade-off_JSD_mt} is shown in Appendix \ref{sec:proof-random-app}.

\begin{prop}
\label{prop:random-sigma}
If $\sigma_\epsilon^2\sqrt{\sum_{i=1}^n\frac{1}{\sigma_{i}^4}} \leq 1$, then we have
\begin{itemize}
\item The Bayesian privacy leakage $\epsilon_{p}$ is at least $C_1 - \frac{C_3}{K} \sigma_\epsilon^2\sqrt{\sum_{i=1}^n\frac{1}{\sigma_{i}^4}} $, which is a decreasing function of $\sigma_\epsilon$.
\item The utility loss $\epsilon_u$ is at most $C_4\sigma_\epsilon^2\sqrt{\sum_{i=1}^n\frac{1}{\sigma_{i}^4}}$, which is an increasing function of $\sigma_\epsilon$.
\end{itemize}
\end{prop}
Proposition \ref{prop:random-sigma} demonstrates that when the variance of adding noise $\sigma_\epsilon$ decreases, the lower bound of Bayesian privacy leakage $\epsilon_{p}$ increases. 
In particular, when $\sigma_\epsilon = 0$, $\sigma_\epsilon^2\sqrt{\sum_{i=1}^n\frac{1}{\sigma_{i}^4}} =0$.
Therefore, $ \epsilon_u=0$ and $\epsilon_{p} \geq C_1$. As a result, at least $C_1$ privacy leakage might be incurred due to the exposure of unprotected model information to adversaries who may launch Bayesian inference attacks. This theoretical analysis is in accordance with the empirical evidence demonstrated in \cite{zhu2020deep,he2019model}.

\subsubsection{Sparsity Mechanism}
\hfill

The trade-offs between privacy and utility change with distinct extents of sparsity. In SFL, clients launch the sparsity mechanism by uploading the partial parameters to the server. Specifically, each client uploads $d$ of total $n$ dimensions to the server. Let $W_k^{\calRO}\in \mathcal W^{\calRO}_k$ represent the parameter sampled from distribution $P^{\calRO}_k = \calN(\mu_0,\Sigma_0)$, where $\mu_0 = (\mu_u, \mu_o), \mu_u \in\mathbb R^d, \mu_o \in\mathbb R^{n-d}$ and $\Sigma_0 = \text{diag}(\Sigma_u^{d\times d}, \Sigma_o^{(n-d)\times (n-d)})$ is a diagonal matrix. Without loss of generality, we assume each client uploads the first $d$ dimensions to the server. We assume the vector composed by the last $(n-d)$ dimensions follows a Gaussian distribution denoted as $\calN(\mu_g, \Sigma_g)$, where $\Sigma_g$ is a diagonal matrix. In this setting, the protected model information follows $P_k^\calD \sim \calN(\mu, \Sigma)$, where $\mu = (\mu_u, \mu_g), \Sigma = \text{diag}(\Sigma_u^{d\times d}, \Sigma_g^{(n-d)\times (n-d)})$. Then server aggregates all the uploaded parameters of clients as $W_a^{\calD}=\frac{1}{K}\sum_{k=1}^KW_k^\calD$ following the distribution $\calN(\mu,\Sigma/K)$, and $W_a^\calRO = \frac{1}{K}\sum_{k=1}^KW_k^\calRO$ follows the distribution $\calN(\mu_0,\Sigma_0/K)$. The following theorem illustrates bounds for privacy leakage and utility loss using the dimension of the uploaded model information $d$. The full proof is deferred to Appendix \ref{sec:proof-sparisty-app}.

\begin{thm}  \label{thm: privacy-utility-tradeoff-sparsity}
For the sparsity mechanism by uploading partial information to the server, denote $h(\mu_1, \Sigma_1, \mu_2, \Sigma_2) =\left(1-\frac{\text{det}(\Sigma_1)^{1/4}\text{det}(\Sigma_2)^{1/4}}{\text{det}\left(\frac{\Sigma_1+\Sigma_2}{2}\right)^{1/2}}\exp\{-\frac{1}{8}(\mu_1-\mu_2)^T(\frac{\Sigma_1+\Sigma_2}{2})^{-1}(\mu_1-\mu_2)\}\right)^{1/2}$, we have

\begin{align}
C_1 \leq \epsilon_p + C_3\cdot h(\mu_o, \mu_g, \Sigma_o, \Sigma_g),
\end{align}
and
\begin{align}
    \epsilon_u \leq C\cdot h(\mu_o, \mu_g, \Sigma_o, \Sigma_g),
\end{align}
where $C_1 = \frac{1}{K}\sum_{k=1}^K \sqrt{{\text{JS}}(F^{\calRO}_k || F^{\calO}_k)}, C_3 =\frac{\sqrt{2}(e^{2\xi}-1)}{2}$ are two constants independent of the protection mechanisms adopted, and $C$ is a  constant satisfying that $\sqrt{2}U(w,d) \leq C$ for any $w \in \mathcal{W}_k$ and $d \in \calS_k$.
\end{thm}



\begin{prop} \label{prop:sparsity-prop}
For the sparsity mechanism by uploading partial information to the server, we have: 
\begin{itemize}
    \item The Bayesian privacy leakage $\epsilon_{p}$ is at least $C_1 - C_3\cdot h(\mu_o, \mu_g, \Sigma_o, \Sigma_g)$, which is an increasing function of the dimension of uploaded model information $d$.
    \item The utility loss $\epsilon_u$ is at most $ C\cdot h(\mu_o, \mu_g, \Sigma_o, \Sigma_g)$, which is a decreasing function of the dimension of uploaded model information $d$.
\end{itemize}
\end{prop}

Proposition \ref{prop:sparsity-prop} demonstrates that when the dimension of uploaded model information $d$ increases, the lower bound of Bayesian privacy leakage $\epsilon_{p}$ increases, and the upper bound of utility loss $\epsilon_u$ decreases. In particular, $d= n$ implies that the server uploads all parameters to the server, then $ h(\mu_o, \mu_g, \Sigma_o, \Sigma_g)= 0$. Therefore, $\epsilon_p \geq C_1$ and $\epsilon_u = 0$. It shows that at least $C_1$ Bayesian privacy leakage occurs, but the utility is not sacrificed.

\subsubsection{Homomorphic Encryption} 
\hfill \\
Homomorphic Encryption (HE) \cite{gentry2009fully,rivest1978data,gentry2010toward} allows certain computation (e.g., addition) to be performed directly on ciphertexts, without decrypting them first. Such characteristic allows applying HE into FL to protect privacy because the server only accesses the uploaded parameters of users on ciphertexts instead of clients’ plaintext directly \cite{zhang2020batchcrypt,zhang2019pefl, aono2017privacy, truex2019hybrid}. Specifically, in federated learning, the uploaded model weights or gradients have been encrypted by the clients themselves. Therefore, it is hard for adversaries to espy the exposed information and infer clients' private data. The following theorem evaluates the privacy leakage and utility loss of the \textbf{approximate eigenvector} method \cite{gentry2013homomorphic}, which is a widely-used method of HE.


\begin{thm} \label{thm:tradeoff-HE}
For the encryption mechanism that encrypts the model information ($W_k^\calD = \text{Enc}(W_k^\calRO)$ using \textbf{approximate eigenvector} method), 
\begin{itemize}
    \item If the private key is unknown for server (adversary), then $\epsilon_p = 0$ and $\epsilon_u \geq \frac{C_1}{C_2}$. 
    \item If the private key is known for server (adversary), then $\epsilon_u = 0$ and $\epsilon_p \geq C_1$.
\end{itemize}
\end{thm}
\textbf{Remark:}
 1) On one hand, if private key is unknown, the Bayesian privacy leakage $\epsilon_p$ is zero because encryption satisfies semantic security, with which $D_k$ is independent of $W_k^{\calD}$. Moreover, since the private key of HE is unknown and the global model $\text{Enc}(W_a^{\calRO})$ remains encrypted, the utility evaluated by server is far away from the optimal utility (utility loss is large). \\
2) On the other hand, if the private key is mistakenly disclosed to server (adversary), then the server could decrypt the global model and utility loss $\epsilon_u=0$ in the ideal case. Moreover, the privacy leakage $\epsilon_p$ is large.

\subsection{Secret Sharing}
\textbf{Secret Sharing } \cite{SecShare-Adi79,SecShare-Blakley79,bonawitz2017practical} were developed to distribute a secret among a group of participants. However, it requires extensive exchange of messages and entails a communication overhead not viable in many federated learning settings. Let $W_k^{\calRO}$ represent the original model information that follows a uniform distribution over $[c_k^1 - \delta, c_k^1 + \delta]\times [c_k^2 - \delta, c_k^2 + \delta]\cdots\times [c_k^n - \delta, c_k^n + \delta]$, $W_k^{\calD}$ represent the distorted model information that follows a uniform distribution over $[c_k^1 - a_k^1, c_k^1 + b_k^1]\times [c_k^2 - a_k^2, c_k^2 + b_k^2]\cdots\times [c_k^n - a_k^{n}, c_k^n + b_k^{n}]$, and $0< \delta< a_k^{i}, b_k^{i}$, $\forall i = 1,2, \cdots, n$. The following theorem measures utility loss and provides lower bounds for privacy leakage. 
\begin{thm}
For the secret sharing mechanism, the privacy leakage
\begin{align}
    \epsilon_{p,k} &\ge \sqrt{{\text{JS}}(F^{\calRO}_k || F^{\calO}_k)} - \frac{1}{2}(e^{2\xi}-1)\cdot\left(1 - \prod_{j = 1}^{m}\left(\frac{2\delta}{b_k^j + a_k^j}\right)\right).
\end{align}
Furthermore, we have that
\begin{align}
    \epsilon_u = 0.
\end{align}
\end{thm}

\textbf{Remark:} For secret sharing mechanism, there does not exist a continuous trade-off between privacy and utility, but \pref{eq: total_variation-privacy trade-off} in \pref{thm: utility-privacy trade-off_JSD_mt} still holds.

%% file: sections/05-discussion.tex
\section{Discussion}\label{sect:compare-BP}

The section illustrates the comparison of the proposed Bayesian Privacy with respect to Differential Privacy \cite{dwork2014algorithmic}, Bayesian Differential Privacy BDP \cite{BDP-icml20}, Local Differential Privacy \cite{duchi2013local} and Information Privacy \cite{du2012privacy}. A brief summary of the comparison is summarized in Table \ref{tab:tradeoff-comp}.

\begin{table}[htbp]
\footnotesize
\begin{tabular}{|c|c|c|c|c|}
\hline
& Def. 
& \begin{tabular}[c]{@{}c@{}}Protection\\ mech.\end{tabular} &
\begin{tabular}[c]{@{}c@{}} Privacy-Utility\\Trade-off \end{tabular} &
\begin{tabular}[c]{@{}c@{}}Applied \\ to FL \end{tabular} \\ \hline
\begin{tabular}[c]{@{}c@{}} DP \cite{dwork2014algorithmic}\end{tabular}
& \begin{tabular}[c]{@{}c@{}}$\frac{\Pr[M(d) \in W]}{\Pr[M(S')\in W]} \in [e^{-\epsilon},e^\epsilon]$\\ for any \textit{adjacent} data $S,S'$ \end{tabular} &\begin{tabular}[c]{@{}c@{}} Add noise \& \\ randomized response \end{tabular}& \begin{tabular}[c]{@{}c@{}} error $n \in \Omega(\frac{1}{\epsilon})$ 
\\ (Theo. 8.7 \cite{dwork2014algorithmic})\end{tabular} 
 & \cite{seif2020wireless,truex2020ldp} \\ \hline
\begin{tabular}[c]{@{}c@{}} LDP\cite{duchi2013local}\end{tabular} 
&\begin{tabular}[c]{@{}c@{}} 
$\sup \Big\{\frac{Q(S|x)}{Q(S|x')}\Big\} \le \exp(\alpha)$
\\ for any data $x,x'$ \end{tabular} 
& \begin{tabular}[c]{@{}c@{}} Add noise \& \\ randomized response \end{tabular}
&\begin{tabular}[c]{@{}c@{}} $c_l\min\{1,\frac{1}{\sqrt{n\alpha^2}},\frac{d}{n\alpha^2} \}\le \mathfrak{M}_n $ \\ (Theo.1 \cite{duchi2013local}) \end{tabular}
& \begin{tabular}[c]{@{}c@{}} \cite{kim2021federated,truex2020ldp} \\ \cite{zhao2020local,seif2020wireless} \end{tabular} \\ \hline
\begin{tabular}[c]{@{}c@{}} BDP \cite{BDP-icml20}\end{tabular} 
& \begin{tabular}[c]{@{}c@{}}$ \Pr \big[\log \frac{P(w|d)}{P(W|S')}\ge \epsilon_\mu \big]\le \delta_\mu$\\ for any \textit{adjacent} data $S,S'$ \\
differing by point $x' \sim \mu(x)$\end{tabular}
& \begin{tabular}[c]{@{}c@{}} Add noise \& \\ randomized response \end{tabular} 
& N.A. 
&\cite{FL-BDP19} \\ \hline
\begin{tabular}[c]{@{}c@{}} IP \cite{du2012privacy}\end{tabular} 
& \begin{tabular}[c]{@{}c@{}}$\frac{p_{S|U}(s|u)}{p_{S}(d)} \in [e^{-\epsilon},e^\epsilon]$\\ for any $s$ and any $u$\end{tabular}
&\begin{tabular}[c]{@{}c@{}} Add noise \\ (e.g., Laplacian) \end{tabular} 
&\begin{tabular}[c]{@{}c@{}}min IP s.t. distortion \\ (Def. 3 \cite{du2012privacy})\end{tabular} 
&no \\ \hline
\begin{tabular}[c]{@{}c@{}} Bayesian \\ Privacy \end{tabular} 
& \begin{tabular}[c]{@{}c@{}} $\Big({JS(F^{\calA}_k || F^{\calO}_k)}\Big)^{\frac{1}{2}}\le \epsilon$, 
\\ for any threat model (Def. \ref{defi: average_privacy_JSD})
\end{tabular} 
& 
\begin{tabular}[c]{@{}c@{}}Add noise, Sparsity \& \\ 
HE (Sect.\ref{sec:application})\\ \end{tabular} 
& \begin{tabular}[c]{@{}c@{}}$ C_1 \le\epsilon_{p} + C_2\epsilon_{u}$ 
\\No free lunch theo. \end{tabular} 
& \begin{tabular}[c]{@{}c@{}} this \\ article \end{tabular}
\\ \hline
\end{tabular}
\caption{Comparison of privacy-utility trade-offs with Bayesian Privacy (BP) and related privacy measurements, namely, Diff. Privacy (DP) \cite{dwork2014algorithmic}, Local DP (LDP) \cite{duchi2013local}, Bayesian DP (BDP) \cite{BDP-icml20} and Information Privacy (IP) \cite{du2012privacy}. \label{tab:tradeoff-comp}}
\end{table}

\subsection{The comparison with Differential Privacy and its variants}
\label{subsect:compare-BP-DP}

We illustrate the comparison of Bayesian Privacy with respect to Differential Privacy and its variants as follows. 
\begin{itemize}
\item First, the privacy-utility trade-off analyzed in differential privacy (DP) is essentially characterized by a reciprocal relation as shown by Theorem 8.7 of \cite{dwork2014algorithmic}, which dedicates that the outputs protected by an $\epsilon$-differentially private mechanism have the maximum error $\Omega\left(\frac{1}{\epsilon}\right)$. This trade-off analysis is in contrast to the \textit{No free lunch theorem} (\pref{thm: utility-privacy trade-off_JSD_mt}) in the present article, which dictates that the sum of privacy loss and utility loss is greater than a problem-dependent constant.

\item Second, it is shown that a $\xi$-maximum Bayesian privacy preserving mapping $f_{W|S}(\cdot)$ is $(2\xi)$-differentially private (see Appendix \ref{appendix:proof-BP-DP} for details). 

\item Third, the privacy-utility trade-off established by the local differential privacy dedicates that, for an $\alpha${-private} estimator, its \textit{minimax error bound} (in utility) is greater than a constant that is proportional to $\frac{1}{\alpha^2}$ (see Theorem 1 in \cite{duchi2013local}). 

\item Fourth, Bayesian Differential Privacy (BDP)  was also proposed to measure the privacy loss following Bayesian rule \cite{duchi2013local}. Nevertheless, BDP estimated the posterior of query outputs (denoted by $S$) given private data (denoted by $x$). This is in sharp contrast to the proposed Bayesian Privacy, which models adversaries' posterior belief of private data (denoted by $S$) given publicly exchanged information (denoted by $W^\calD$). Moreover, no privacy-utility trade-off analysis was provided for BDP in \cite{duchi2013local}. 

\item Last but not least, DP and its variants are not amenable to include encryption as a protection mechanism. To our best knowledge, the SFL framework based on Bayesian inference attack is the first that is applicable to different privacy-protection mechanisms, including randomization in Differential Privacy \cite{geyer2017differentially,truex2020ldp,abadi2016deep}, Sparsity \cite{shokri2015privacy, gupta2018distributed, thapa2020splitfed}, Homomorphic Encryption \cite{gentry2009fully,batchCryp} and Secret Sharing \cite{SecShare-Adi79,SecShare-Blakley79,bonawitz2017practical}.

\end{itemize}

\subsection{The relationship between Bayesian Privacy and Information Privacy}\label{subsect:compare-BP-IP}

There are three fundamental differences between Bayesian Privacy (BP) and Information Privacy (IP) \cite{du2012privacy} as shown below. 
\begin{itemize}
\item One fundamental difference between BP and IP lies in the fact that BP is formulated within a \textit{secure federated learning} setting as illustrated in Sect. \ref{sec:framework}, while it is unclear how to apply IP for delineating the privacy-utility trade-off in distributed learning scenarios. 

\item As per the privacy definitions, BP involves the \textit{averaged} gain of information about private data that an adversary may obtain by observing publicly exposed model information (see Def. \ref{defi: average_privacy_JSD} in this article). On the other hand, information privacy is defined based on the \textit{bound} of such information gain (see Def. 6 in \cite{du2012privacy}). 
Note that we also give the bounds of Bayesian Privacy in Eq. (\ref{eq:ep-bayes-privacy}), which is nevertheless based on the average gain. 

\item In terms of the privacy-utility trade-off, the optimal trade-off in BP is cast as the solution of a constrained optimization where the goal is to minimize utility loss subject to the hard constrain of maintaining Bayesian privacy leakage less than an acceptable threshold (the problem in Eq. (\ref{eq: key_constraint_upper_bound})). In contrast, the optimal trade-off in IP is to minimize the privacy leakage subject to a given utility loss constrain (Def. 3 in \cite{du2012privacy}). 
\end{itemize}

\textbf{Remark:}
It is worth noting that the Bayesian privacy is measured using JS divergence instead of KL divergence. Unlike KL divergence, {\text{JS}} divergence is symmetrical and its square root satisfies the triangle inequality. The property of triangular inequality of JS divergence facilitates the derivation of our main results illustrated in Sect. \ref{subsect:PU-tradeoff}.

%% file: sections/06-conclusion.tex
\section{Conclusion}
\label{sec:conclusion}


In this article, we proposed a statistical framework in SFL based on Bayesian inference attack to take into account two compelling requirements i.e. the protection of private data as well as the maximization of model utility in secure federated learning (SFL). The proposed framework is based on the cornerstone definition of Bayesian privacy Leakage that quantitatively measures the amount of information about private data that semi-honest adversaries may infer from the espied information during the SFL learning process. 
Then the trade-off privacy and utility is 
formulated as an optimization problem in which the goal is to find a modified distribution of exchanged model information that achieves the smallest utility loss without incurring privacy loss above an acceptable level. The theoretical analysis of the trade-off then leads us to the No Free Lunch (NFL) theorem for federated learning, which dictates that one has to trade a high privacy persevering guarantee with a certain degree of utility loss, and vice versa. This trade-off analysis is applicable to protection mechanisms including Randomization \cite{geyer2017differentially,truex2020ldp,abadi2016deep}
Sparsity \cite{shokri2015privacy, gupta2018distributed, thapa2020splitfed}, Homomorphic Encryption \cite{gentry2009fully,batchCryp} and Secret Sharing \cite{SecShare-Adi79,SecShare-Blakley79,bonawitz2017practical} which are special cases of the general theorem proved in \pref{thm: utility-privacy trade-off_JSD_mt}. 

The analysis of privacy-utility trade-off in secure federated learning (SFL), to our best knowledge, is the first work that takes into account different specific protection mechanisms, including DP and HE, in a unified framework. While the main result disclosed in this article is inspiring, we plan to apply the technique we developed during this research endeavor to explore in the future the trade-off between efficiency and privacy, which provides a more comprehensive view of secure federated learning. \\


%% file: Appendix/AppendixA.tex
\textbf{\Huge{Appendix}}

\section{Theoretical Analysis using {\text{JS}} Divergence}
\label{sec:BP}

Before introducing the analysis in detail, we first illustrate the key characteristics of the metric for measuring such a relationship.
\begin{property}\label{property: TriangleiIneq}
The square root of the Jensen-Shannon divergence satisfies the triangle inequality. Specifically, 
    \begin{align*}
        \sqrt{{\text{JS}}(F^{\calRO}_k || F^{\calO}_k)} - \sqrt{{\text{JS}}(F^{\calA}_k || F^{\calO}_k)}\le \sqrt{{\text{JS}}(F^{\calA}_k || F^{\calRO}_k)}.
    \end{align*}
\end{property}

\begin{property}\label{property: derive_using_AM-GM_inequality}
From AM-GM inequality, we have
\begin{align*}
    \frac{f^{\calA}_{D_k}(d)}{f^{\calM}_{D_k}(d)}\le \frac{f^{\calM}_{D_k}(d)}{f^{\calO}_{D_k}(d)}.
\end{align*}
\end{property}

\subsection{The quantitative relationship between $\text{TV}(P^{\calD}_k || P_k^{\calRO})$ and $\epsilon_{p,k}$}

\begin{lem}\label{lem: JSBound}
Let $P_k^{\calRO}$ and $P^{\calD}_k$ represent the distribution of the parameter of client $k$ before and after being protected. Let $F^{\calA}_k$ and $F^{\calRO}_k$ represent the belief of client $k$ about $D$ after observing the protected and original parameter. Then we have
\begin{align*}
{\text{JS}}(F^{\calA}_k || F^{\calRO}_k)\le \frac{1}{4}(e^{2\xi}-1)^2{\text{TV}}(P_k^{\calRO} || P^{\calD}_k)^2. 
\end{align*}
\end{lem}

\begin{proof}

Let $F^{\calM}_k = \frac{1}{2}(F^{\calA}_k+ F^{\calRO}_k)$. We have

\begin{align*}
{\text{JS}}(F^{\calA}_k || F^{\calRO}_k) & = \frac{1}{2}\left[KL\left(F^{\calA}_k, F^{\calM}_k\right) + KL\left(F^{\calRO}_k,F^{\calM}_k\right)\right]\\
& = \frac{1}{2}\left[\int_{\mathcal{D}_k} f^{\calA}_{D_k}(d)\log\frac{f^{\calA}_{D_k}(d)}{f^{\calM}_{D_k}(d)}\textbf{d}\mu(d) + \int_{\mathcal{D}_k} f^{\calRO}_{D_k}(d)\log\frac{f^{\calRO}_{D_k}(d)}{f^{\calM}_{D_k}(d)}\textbf{d}\mu(d)\right]\\
& = \frac{1}{2}\left[\int_{\mathcal{D}_k} f^{\calA}_{D_k}(d)\log\frac{f^{\calA}_{D_k}(d)}{f^{\calM}_{D_k}(d)}\textbf{d}\mu(d) - \int_{\mathcal{D}_k} f^{\calRO}_{D_k}(d)\log\frac{f^{\calM}_{D_k}(d)}{f^{\calRO}_{D_k}(d)}\textbf{d}\mu(d)\right]\\
&\le \frac{1}{2}\int_{\mathcal{D}_k}\left|f^{\calA}_{D_k}(d) - f^{\calRO}_{D_k}(d)\right|\left|\log\frac{f^{\calM}_{D_k}(d)}{f^{\calRO}_{D_k}(d)}\right|\textbf{d}\mu(d),
\end{align*}
where the inequality is due to $\frac{f^{\calA}_{D_k}(d)}{f^{\calM}_{D_k}(d)}\le \frac{f^{\calM}_{D_k}(d)}{f^{\calRO}_{D_k}(d)}$.

\textbf{Bounding $\left|f^{\calA}_{D_k}(d) - f^{\calRO}_{D_k}(d)\right|$.}

Let $\mathcal U_k = \{w\in\mathcal W_k: dP^{\calD}_k(w) - dP_k^{\calRO}(w)\ge 0\}$, and $\mathcal V_k = \{w\in\mathcal W_k: dP^{\calD}_k(w) - dP_k^{\calRO}(w)< 0\}$.

Then we have 

\begin{align}\label{eq:initial_step_{JS}}
    \left|f^{\calA}_{D_k}(d) - f^{\calRO}_{D_k}(d)\right| &= \left|\int_{\mathcal W_k} f_{D_k|W_k}(d|w)[d P^{\calD}_k(w) - d P_k^{\calRO}(w)]\right|\nonumber\\
    &= \left|\int_{\mathcal{U}_k} f_{D_k|W_k}(d|w)[d P^{\calD}_k(w) - d P_k^{\calRO}(w)] + \int_{\mathcal{V}_k} f_{D_k|W_k}(d|w)[d P^{\calD}_k(w) - d P_k^{\calRO}(w)]\right|\nonumber\\
    &\le\left(\sup_{w\in\mathcal{W}_k} f_{D_k|W_k}(d|w) - \inf_{w\in\mathcal{W}_k} f_{D_k|W_k}(d|w)\right)\int_{\mathcal{U}_k} [d P^{\calD}_k(w) - d P_k^{\calRO}(w)].
\end{align}



Notice that

\begin{align*}
    \sup_{w\in\mathcal W_k} f_{D_k|W_k}(d|w) - \inf_{w\in\mathcal W_k} f_{D_k|W_k}(d|w) = \inf_{w\in\mathcal W_k} f_{D_k|W_k}(d|w)\left|\frac{\sup_{w\in\mathcal W_k} f_{D_k|W_k}(d|w)}{\inf_{w\in\mathcal W_k} f_{D_k|W_k}(d|w)}-1\right|.
\end{align*}

From the definition of $\xi$, we know that for any $w\in\mathcal W_k$,
\begin{align*}
    e^{-\xi}\le\frac{f_{D_k|W_k}(d|w)}{f_{D_k}(d)}\le e^{\xi},
\end{align*}

Therefore, for any pair of parameters $w,w'\in\mathcal W_k$, we have
\begin{align*}
    \frac{f_{D_k|W_k}(d|w)}{f_{D_k|W_k}(s|w')} = \frac{f_{D_k|W_k}(d|w)}{f_{D_k}(d)}/\frac{f_{D_k|W_k}(s|w')}{f_{D_k}(d)}\le e^{2\xi}. 
\end{align*}

Therefore, the first term of \pref{eq:initial_step_{JS}} is bounded by

\begin{align}\label{eq: bound_1_term_1_{JS}_ratio}
    \sup_{w\in\mathcal W_k} f_{D_k|W_k}(d|w) - \inf_{w\in\mathcal W_k} f_{D_k|W_k}(d|w) \le \inf_{w\in\mathcal W_k} f_{D_k|W_k}(d|w)(e^{2\xi}-1).
\end{align}

From the definition of total variation distance, we have
\begin{align}\label{eq: bound_1_term_2_{JS}_ratio}
    \int_{\UU} [d P^{\calD}_k(w) - d P_k^{\calRO}(w)] = {\text{TV}}(P_k^{\calRO} || P^{\calD}_k).
\end{align}

Combining \pref{eq: bound_1_term_1_{JS}_ratio} and \pref{eq: bound_1_term_2_{JS}_ratio}, we have
\begin{align}\label{eq: bound_for_the_gap}
        |f^{\calA}_{D_k}(d) - f^{\calRO}_{D_k}(d)| &=\left(\sup_{w\in\mathcal W_k} f_{D_k|W_k}(d|w) - \inf_{w\in\mathcal W_k} f_{D_k|W_k}(d|w)\right)\int_\UU [d P^{\calD}_k(w) - d P_k^{\calRO}(w)]\nonumber\\
        &\le\inf_{w\in\mathcal W_k} f_{D_k|W_k}(d|w)(e^{2\xi}-1){\text{TV}}(P_k^{\calRO} || P^{\calD}_k).
\end{align}


\textbf{Bounding $\left|\log\left(\frac{f^{\calM}_{D_k}(d)}{ f^{\calRO}_{D_k}(d)}\right)\right|.$}

We have that

\begin{align}\label{eq: bound_for_log_ratio}
    \left|\log\frac{f^{\calM}_{D_k}(d)}{f^{\calRO}_{D_k}(d)}\right|&\le\frac{|f^{\calM}_{D_k}(d) - f^{\calRO}_{D_k}(d)|}{\min\{f^{\calM}_{D_k}(d), f^{\calRO}_{D_k}(d)\}}\nonumber\\
    &=\frac{|f^{\calA}_{D_k}(d) - f^{\calRO}_{D_k}(d)|}{2\min\{f^{\calM}_{D_k}(d), f^{\calRO}_{D_k}(d)\}}\nonumber\\
    &\le \frac{\inf_{w\in\mathcal W_k} f_{D_k|W_k}(d|w)(e^{2\xi}-1){\text{TV}}(P_k^{\calRO} || P^{\calD}_k)}{2\min\{f^{\calM}_{D_k}(d), f^{\calRO}_{D_k}(d)\}}\nonumber\\
    &\le \frac{1}{2}(e^{2\xi}-1){\text{TV}}(P_k^{\calRO} || P^{\calD}_k),
\end{align}
where the first inequality is due to \pref{lem: log_upper_bound}, the third inequality is due to $\min\{f^{\calM}_{D_k}(d), f^{\calRO}_{D_k}(d)\}\ge \min\{f^{\calA}_{D_k}(d), f^{\calRO}_{D_k}(d)\}\ge \inf\limits_{\small{w\in\mathcal W_k}} f_{D_k|W_k}(d|w)$.

Combining \pref{eq: bound_for_the_gap} and \pref{eq: bound_for_log_ratio}, we have
\begin{align*}
    {\text{JS}}(F^{\calA}_k || F^{\calRO}_k) & \le  \frac{1}{2}\left[\int_{\mathcal{D}_k} \left|(f^{\calA}_{D_k}(d) - f^{\calRO}_{D_k}(d))\right| \left|\log\frac{f^{\calM}_{D_k}(d)}{f^{\calRO}_{D_k}(d)}\right|\textbf{d}\mu(d)\right]\\
    &\le \frac{1}{4}(e^{2\xi}-1)^2{\text{TV}}(P_k^{\calRO} || P^{\calD}_k)^2\int_{\mathcal{D}_k} \inf_{w\in\mathcal W^{\calRO}_k} f_{D_k|W_k}(d|w)\textbf{d}\mu(d)\\
    &\le\frac{1}{4}(e^{2\xi}-1)^2{\text{TV}}(P_k^{\calRO} || P^{\calD}_k)^2.
\end{align*}

\end{proof}


\begin{lem}\label{lem: total_variation-privacy trade-off_appendix}
Let $\epsilon_{p,k}$ be defined in \pref{defi: average_privacy_JSD}. Let $P_k^{\calRO}$ and $P^{\calD}_k$ represent the distribution of the parameter of client $k$ before and after being protected. Let $F^{\calB}_k$ and $F^{\calA}_k$ represent the belief of client $k$ about $D$ before and after observing the released parameter. Then we have

\begin{align}\label{eq: total_variation-privacy trade-off-app}
    \sqrt{{\text{JS}}(F^{\calRO}_k || F^{\calO}_k)} \le\epsilon_{p,k} + \frac{1}{2}(e^{2\xi}-1){\text{TV}}(P_k^{\calRO} || P^{\calD}_k).
\end{align}
Consequently, we have
\begin{align*}
    \frac{1}{K}\sum_{k=1}^K \sqrt{{\text{JS}}(F^{\calRO}_k || F^{\calO}_k)} \le\epsilon_p + \frac{1}{K}\sum_{k=1}^K \frac{1}{2}(e^{2\xi}-1){\text{TV}}(P_k^{\calRO} || P^{\calD}_k).
\end{align*}
\end{lem}

\begin{proof}
From \pref{lem: JSBound}, we know that 
\begin{align*}
{\text{JS}}(F^{\calA}_k || F^{\calRO}_k)\le\frac{1}{4}(e^{2\xi}-1)^2{\text{TV}}(P_k^{\calRO} || P^{\calD}_k)^2. 
\end{align*}

This further implies that
\begin{align*}
    \sqrt{{\text{JS}}(F^{\calA}_k || F^{\calRO}_k)}\le \frac{1}{2}(e^{2\xi}-1){\text{TV}}(P_k^{\calRO} || P^{\calD}_k).
\end{align*}

Notice that the square root of the Jensen-Shannon divergence satisfies the triangle inequality. 
Then we have that
\begin{align*}
    \sqrt{{\text{JS}}(F^{\calRO}_k || F^{\calO}_k)} - \sqrt{{\text{JS}}(F^{\calA}_k || F^{\calO}_k)}\le \sqrt{{\text{JS}}(F^{\calA}_k || F^{\calRO}_k)}\le \frac{1}{2}(e^{2\xi}-1){\text{TV}}(P_k^{\calRO} || P^{\calD}_k).
\end{align*}
Therefore, we have 
\begin{align*}
    \sqrt{{\text{JS}}(F^{\calRO}_k || F^{\calO}_k)} &\le \sqrt{{\text{JS}}(F^{\calA}_k || F^{\calO}_k)} + \frac{1}{2}(e^{2\xi}-1){\text{TV}}(P_k^{\calRO} || P^{\calD}_k)\\
    & = \epsilon_{p,k} + \frac{1}{2}(e^{2\xi}-1){\text{TV}}(P_k^{\calRO} || P^{\calD}_k).
\end{align*}
\end{proof}

\begin{lem}[\cite{duchi2013local}]\label{lem: log_upper_bound}
For two positive numbers $a$ and $b$, we have that
\begin{align*}
 \left|\log\left(\frac{a}{b}\right)\right| \le \frac{|a-b|}{\min\{a,b\}}.
\end{align*}
\end{lem}

\subsection{The quantitative relationship between ${\text{TV}}(P^{\calRO}_a || P^{\calD}_a )$ and $\epsilon_u$}

\begin{lem}\label{lem: total_variation-utility trade-off}
Let \pref{assump: assump_of_Delta} hold, and $\epsilon_{u}$ be defined in \pref{defi: utility_loss}. Let $P_a^{\calRO}$ and $P_a^{\calD}$ represent the distribution of the aggregated parameter before and after being protected. Then we have,
\begin{align*}
    \epsilon_{u} \ge\frac{\Delta}{2}\cdot {\text{TV}}(P^{\calRO}_a || P^{\calD}_a ).
\end{align*}
\end{lem}


\begin{proof}

Let $\mathcal U_a = \{w\in\mathcal W_a: dP_a^{\calD}(w) - dP_a^{\calRO}(w)\ge 0\}$, and $\mathcal V_a = \{w\in\mathcal W_a: dP_a^{\calD}(w) - dP_a^{\calRO}(w)< 0\}$, where $\mathcal W_a$ represents the union of the supports of $P_a^{\calD}$ and $P_a^{\calRO}$. 

For any $w\in\mathcal V_a$, the definition of $\mathcal V_a$ implies that $dP_a^{\calRO}(w) > dP_a^{\calD}(w)\ge 0$. Therefore, $w$ belongs to the support of $P_a^{\calRO}$, which is denoted as $\mathcal W^{\calRO}_a$. Therefore we have that
\begin{align}\label{eq: subset_relationship_1}
    \mathcal V_a\subset\mathcal W^{\calRO}_a.
\end{align}

Similarly, we have that
\begin{align}\label{eq: subset_relationship_n}
    \mathcal U_a\subset\mathcal W^{\calD}_a.
\end{align}


It is assumed that the utility of the aggregated model information achieves the maximal value at the convergence step \cite{li2019convergence, haddadpour2019convergence}. Therefore, we have that
\begin{align}\label{eq: subset_relationship_2}
    \mathcal W^{\calRO}_a \subset \mathcal W^{*}_a.
\end{align}

Notice that from the definition of $\mathcal W^{*}_a$, for any $w\in\mathcal W_a$ and $w^*\in\mathcal W_a^*$ we have that
\begin{align}\label{eq: w_a_best}
     \frac{1}{K}\sum_{k=1}^K U_k(w^{*})\ge \frac{1}{K}\sum_{k=1}^K U_k(w).
\end{align}

Let $\Delta$ be a positive constant defined in \pref{assump: assump_of_Delta}, from \pref{defi: neighbor_set} we have
$$\calW_{\Delta} = \left\{w\in\calW^{\calD}_a: \left|\frac{1}{K}\sum_{k=1}^K  U_k(w^{*})-\frac{1}{K}\sum_{k=1}^K U_k(w)\right|\le\Delta, \forall w^{*}\in\mathcal W^{*}_a\right\},$$ which implies that for any $w\in\calW^{\calD}_a\setminus\calW_{\Delta}$ and $w^*\in\mathcal W_a^*$ it holds that
\begin{align}\label{eq:l_U_bound_11}
  \left|\frac{1}{K}\sum_{k=1}^K  U_k(w^{*})-\frac{1}{K}\sum_{k=1}^K U_k(w)\right|>\Delta.   
\end{align}

Combining \pref{eq: w_a_best} and \pref{eq:l_U_bound_11}, for any $w\in\calW^{\calD}_a\setminus\calW_{\Delta}$ and $w^*\in\mathcal W_a^*$ we have 
\begin{align}\label{eq:l_U_bound_1}
\frac{1}{K}\sum_{k=1}^K  U_k(w^{*})-\frac{1}{K}\sum_{k=1}^K U_k(w)>\Delta.
\end{align}

Recall that $W_a$ represents the aggregated parameter, $P_a^{\calD}$ represents the distribution of the aggregated parameter after being protected, and $p^{\calD}_{W_a}(w)$ represents the corresponding probability density function. We denote $U_a(w) = \frac{1}{K}\sum_{k=1}^K U_k(w)$, then we have
\begin{align*}
&\epsilon_{u} =\frac{1}{K}\sum_{k=1}^K \epsilon_{u,k}\\
     &=\frac{1}{K}\sum_{k=1}^K [U_k(P^{\calRO}_{a}) - U_k(P^{\calD}_{a})]\\  
    &=\frac{1}{K}\sum_{k=1}^K\left[\mathbb E_{w\sim P^{\calRO}_{a}}[U_k(w)] - \mathbb E_{w\sim P^{\calD}_{a}}[U_k(w)]\right]\\
     &=\frac{1}{K}\sum_{k=1}^K\left[\int_{\mathcal W_k} U_k(w)dP^{\calRO}_a(w) - \int_{\mathcal W_k} U_k(w) dP^{\calD}_a(w)\right]\\
     &=\frac{1}{K}\sum_{k=1}^K\left[\int_{\mathcal{V}_a} U_k(w)[d P^{\calRO}_{a}(w) - d P^{\calD}_{a}(w)] - \int_{\mathcal{U}_a} U_k(w)[d P^{\calD}_{a}(w) - d P^{\calRO}_{a}(w)]\right]\\
    &\overset{\spadesuit}{=}\frac{1}{K}\sum_{k=1}^K\left[\int_{\mathcal V_a} U_k(w)\one\{w\in\mathcal W^{*}_a\}[d P^{\calRO}_{a}(w) - d P^{\calD}_{a}(w)] - \int_{\mathcal{U}_a} U_k(w)[d P^{\calD}_{a}(w) - d P^{\calRO}_{a}(w)]\right]\\
     &\overset{\star}{=}\int_{\mathcal V_a} U_a(w)\one\{w\in\mathcal W^{*}_a\}[d P^{\calRO}_{a}(w) - d P^{\calD}_{a}(w)] - \int_{\mathcal{U}_a} U_a(w)\one\{w\in\mathcal W^{\calD}_a\}[d P^{\calD}_{a}(w) - d P^{\calRO}_{a}(w)]
\end{align*}
where $\spadesuit$ is due to $\mathcal V_a\subset\mathcal W^{\calRO}_a\subset \mathcal W^{*}_a$ from \pref{eq: subset_relationship_1} and \pref{eq: subset_relationship_2}, and $\star$ is due to $\mathcal U_a\subset\mathcal W^{\calD}_a$ from \pref{eq: subset_relationship_n}.

We decompose $\int_{\mathcal{U}_a} U_a(w)\one\{w\in\mathcal W^{\calD}_a\}[d P^{\calD}_{a}(w) - d P^{\calRO}_{a}(w)]$ as the summation of $\int_{\mathcal{U}_a} U_a(w)\one\{w\in\mathcal W^{\calD}_a\}\one\{w\in\calW_{\Delta}\}[d P^{\calD}_{a}(w) - d P^{\calRO}_{a}(w)]$ and $\int_{\mathcal{U}_a} U_a(w)\one\{w\in\mathcal W^{\calD}_a\}\one\{w\not\in\calW_{\Delta}\}[d P^{\calD}_{a}(w) - d P^{\calRO}_{a}(w)]$. Then we have
\begin{align*}
&\epsilon_{u} =\frac{1}{K}\sum_{k=1}^K \epsilon_{u,k}\\
&=\int_{\mathcal V_a}U_a(w)\one\{w\in\mathcal W^{*}_a\}[d P^{\calRO}_{a}(w) - d P^{\calD}_{a}(w)] - \int_{\mathcal{U}_a}U_a(w)\one\{w\in\mathcal W^{\calD}_a\}[d P^{\calD}_{a}(w) - d P^{\calRO}_{a}(w)]\\
&\ge\Delta\cdot\left[{\text{TV}}(P^{\calRO}_a || P^{\calD}_a ) - \int_{\mathcal{U}_a}\one\{w\in\mathcal W^{\calD}_a\}\one\{w\in\calW_{\Delta}\}[d P^{\calD}_{a}(w) - d P^{\calRO}_{a}(w)] \right] \\
&\ge\Delta\cdot{\text{TV}}(P^{\calRO}_a || P^{\calD}_a ) - \Delta\cdot\int_{\mathcal W^{\calD}_a}\one\{w\in\calW_{\Delta}\} p^{\calD}_{W_a}(w) dw\\
&\ge\frac{\Delta}{2}\cdot {\text{TV}}(P^{\calRO}_a || P^{\calD}_a ),
\end{align*}
in which



\begin{itemize}
\item the first inequality is due to $U_a(w)\le U_a(w^{*})$ for any $w\in\calW_{\Delta}$ and $w^{*}\in\mathcal W^{*}_a$ according to \pref{eq: w_a_best}, and $U_a(w^{*})-U_a(w)>\Delta$ for any $w\in\calW^{\calD}_a\setminus\calW_{\Delta}$ and $w^{*}\in\mathcal W^{*}_a$ from \pref{eq:l_U_bound_1}.
\item the second inequality is due to $\int_{\mathcal{U}_a}\one\{w\in\mathcal W^{\calD}_a\}\one\{w\in\calW_{\Delta}\}[d P^{\calD}_{a}(w) - d P^{\calRO}_{a}(w)]\le\int_{\mathcal{U}_a}\one\{w\in\mathcal W^{\calD}_a\}\one\{w\in\calW_{\Delta}\}d P^{\calD}_{a}(w)\le \int_{\mathcal W^{\calD}_a}\one\{w\in\calW_{\Delta}\}d P^{\calD}_{a}(w) = \int_{\mathcal W^{\calD}_a}\one\{w\in\calW_{\Delta}\} p^{\calD}_{W_a}(w) dw$. 
\item the third inequality is due to $\int_{\mathcal W^{\calD}_a}\one\{w\in\calW_{\Delta}\} p^{\calD}_{W_a}(w) dw\le\frac{{\text{TV}}(P_a^{\calRO} || P_a^{\calD} )}{2}$.

\end{itemize}
\end{proof}

\subsection{Analysis of \pref{thm: utility-privacy trade-off_JSD_mt}}

With \pref{lem: total_variation-privacy trade-off_appendix} and \pref{lem: total_variation-utility trade-off}, it is now natural to provide a quantitative relationship between the utility loss and the privacy leakage (\pref{thm: utility-privacy trade-off_JSD_mt}).

\begin{thm}[No free lunch theorem (NFL) for security and utility]\label{thm: utility-privacy trade-off_JSD} Let $\epsilon_p$ be defined in Def. \ref{defi: average_privacy_JSD}, we have that
\begin{align}\label{eq: total_variation-privacy trade-off_app_1}
C_1\le\epsilon_{p} + \frac{1}{K}\sum_{k=1}^K \frac{1}{2}(e^{2\xi}-1)\cdot {\text{TV}}(P_k^{\calRO} || P^{\calD}_k).
\end{align}

Furthermore, let $\epsilon_u$ be defined in Def. \ref{defi: utility_loss} at the convergence step, with \pref{assump: assump_of_Delta} we have that
\begin{align}\label{eq: total_variation-privacy trade-off_app_2}
 C_1 \le\epsilon_{p} + C_2\cdot \epsilon_{u},
\end{align}
in which
\begin{itemize}
\item $\xi = \max_{k\in [K]} \xi_k$, where $\xi_k = \max_{w\in \mathcal{W}_k, d \in \mathcal{D}_k} \left|\log\left(\frac{f_{D_k|W_k}(d|w)}{f_{D_k}(d)}\right)\right|$ represents the maximum privacy leakage over all possible information $w$ released by client $k$, and $[K] = \{1,2,\cdots, K\}$. $\xi$ is a constant independent of the protection and attack mechanisms;

\item $C_1 = \frac{1}{K}\sum_{k=1}^K \sqrt{{\text{JS}}(F^{\calRO}_k || F^{\calO}_k)}$ is a constant 
representing the averaged square root of JS divergence between adversary's belief distribution about the private information of client $k$ before and after observing the unprotected parameter. This constant is independent of the protection mechanisms.

\item $C_2 = \frac{\gamma}{4\Delta}(e^{2\xi}-1)$ is a constant once the protection mechanisms, the utility function, and the data sets are fixed, where $\gamma = \frac{\sum_{k=1}^K {\text{TV}}(P_k^{\calRO} || P^{\calD}_k)}{{\text{TV}}(P^{\calRO}_a || P^{\calD}_a )}$\footnote{see details of analyzing of the value of $\gamma$ in Sect. \ref{sec:application}}. 
\end{itemize}
\end{thm}

\begin{proof}
From \pref{lem: total_variation-privacy trade-off_appendix}, we have

\begin{align}\label{eq: converge_eq_1_modify_1_0}
    \frac{1}{K}\sum_{k=1}^K \sqrt{{\text{JS}}(F^{\calRO}_k || F^{\calO}_k)} \le\frac{1}{K}\sum_{k=1}^K \epsilon_{p,k} + \frac{1}{K}\sum_{k=1}^K \frac{1}{2}(e^{2\xi}-1){\text{TV}}(P_k^{\calRO} || P^{\calD}_k).
\end{align}

From \pref{lem: total_variation-utility trade-off}, we have
\begin{align}\label{eq: converge_eq_2_0}
    \epsilon_{u} \ge \frac{\Delta}{2}\cdot {\text{TV}}(P^{\calRO}_a || P^{\calD}_a ).
\end{align}

Combining \pref{eq: converge_eq_1_modify_1_0} and \pref{eq: converge_eq_2_0}, we have that

\begin{align*}
    \frac{1}{K}\sum_{k=1}^K \sqrt{{\text{JS}}(F^{\calRO}_k || F^{\calO}_k)} \le\epsilon_{p} + \frac{\gamma}{4\Delta}(e^{2\xi}-1)\epsilon_u,
\end{align*}
where $\gamma = \frac{\frac{1}{K}\sum_{k=1}^K {\text{TV}}(P_k^{\calRO} || P^{\calD}_k)}{{\text{TV}}(P^{\calRO}_a || P^{\calD}_a )}$, and $\epsilon_{p} = \frac{1}{K}\sum_{k=1}^K \epsilon_{p,k}$.

The above equation could be further simplified as
\begin{align*}
    C_1 &\le\epsilon_{p} + C_2\epsilon_u,
\end{align*}
where $C_1 = \frac{1}{K}\sum_{k=1}^K \sqrt{{\text{JS}}(F^{\calRO}_k || F^{\calO}_k)}$ and $C_2 = \frac{\gamma}{4\Delta}(e^{2\xi}-1)$.

\end{proof}

%% file: Appendix/AppendixB.tex
\section{Bayesian Inference Attack In SFL} \label{App:bayes-inference-attack}
It was shown that semi-honest adversaries can recover the private training images with pixel-level accuracy from unprotected gradients of learned models.  We assume that an adversary aims to recover the $k_{th}$ client's private variable $D_k$ from exposed variable $W_k^{\calD}$, which is the output of applying certain protection mechanisms on model information i.e. $W_k^{\calD} = M(W_k^{\calRO})$, where $M$ is a protection mechanism.

In this section, we formulate three privacy leakage attacks of SFL including gradient-inverse attack \cite{zhu2020deep,geiping2020inverting,zhao2020idlg, yin2021see}, model inversion attack \cite{fredrikson2015model, he2019model} and brute-force attack (especially for encryption) from the Bayesian Inference perspective. Such a Bayesian Inference Attack mechanism is defined as follows:

\begin{definition}[Bayesian Inference Attack] \label{def:attackA-app}
A \textbf{Bayesian Inference Attack}\footnote{The Bayesian inference framework has long been applied to the image restoration problem as early as in 1960s~\cite{dempster1968generalization,imgRestorGeman84,box2011bayesian}.} 
is an optimization process that aims to infer the private variable $D_k$ in order to best fit the exposed information $W_k^{\calD}$ i.e., 
\begin{equation} \begin{split} \label{eq:bayes-infer-attack-appendix}
    d^{*} &=\argmax\limits_{d}\log\left(f_{S_k|W_k^{\calD}}(d|w)\right) \\
    & =\argmax\limits_{d}\log\left(\frac{f_{W_k^{\calD}|S_k}(w|d) f_{D_k}(d)}{f_{W_k^{\calD}}(w)}\right) \\
    &=\argmax\limits_{d} [\log f_{W_k^{\calD}|S_k}(w|d)+\log f_{D_k}(d)],\\ 
    & =\argmax\limits_{d}[L_1(d, W_k^{\calD}) +L_2(d)],
\end{split}
\end{equation}
where $L_1(d, W_k^{\calD}) = \log f_{W_k^{\calD}|S_k}(w|d)$ and $L_2(d)=\log f_{D_k}(d)$.
\end{definition}
There are three types of Bayesian Inference attack as follows:
\subsection{Gradient Inversion Attack}
In the SFL, private model updates (gradients) and model \footnote{when updating model, the server could obtain the model gradients by calculating the difference between two uploaded models} needed to be uploaded to the server, which may be taken use to reconstruct the input data of clients by semi-honest adversaries. 
In this setting, gradient-inverse attack is proposed in \cite{zhu2020deep} to infer the private data $D_k$ from the model gradients (exposed information $W_k^{\calD}$ here is model gradients $G$). The form of $L_1(d, W_k^{\calD})$ is:  
\begin{equation} \label{eq:gia-loss}
        L_1(d, W_k^{\calD})= L_1(d, G) = C - \| \nabla W(d) - G\|^2, 
\end{equation}
in which $\nabla W(d)$ is the gradient of training loss w.r.t. model parameters for the estimated data $d$, and $G$ is the observed gradients of model parameters.

\subsubsection{Influence of prior in gradient inversion attack} \label{App:influ-prior}
\hfill \\
Moreover, the prior $L_2(d)=\log f_{D_k}(d)$ is an important factor for Bayesian Inference attack, a series of work was developed to use different prior in gradient inversion attack.
\begin{itemize}
    \item \textbf{No prior:} Zhu \textit{et al.} \cite{zhu2020deep} proposed the gradient inversion attack in which reconstructed private data is randomly initialized. 
    \item \textbf{Smoothness prior:} Geiping \textit{et al.} \cite{geiping2020inverting} improved the attack with the smoothness of data (such as image). Specifically, they viewed the {\text{TV}} loss of the data as the smoothness: $L_2(d) = {\text{TV}}(d)$.
    \item \textbf{Label prior:} Zhao \textit{et. al} \cite{zhao2020idlg} developed the prior with the label information of the data, that is, $L_2(d) = \text{Label}(d)$.
    \item \textbf{Group consistency prior:} Besides smoothness and label of data, Yin et al. \cite{yin2021see} applied group consistency of estimated data to the training process of Gradient inversion attack, where $L_2(d) = \text{Group}(d)$.
\end{itemize}
Moreover, the results of \cite{zhu2020deep,geiping2020inverting,zhao2020idlg,yin2021see} demonstrate that different prior would cause a different type of privacy leakage. Stronger prior of private data let the adversaries attack the private data of clients more effectively. For example, the attack method proposed by \cite{yin2021see} combined three priors: smoothness, label, and group consistency into gradient inversion attack. It is shown that their attack performs best in images with large batch size compared to other gradient inversion methods.

\subsection{Model Inversion Attack} \label{sec:MI}
In some cases of SFL, such as split learning \cite{gupta2018distributed}, or vertical federated learning \cite{yang2019federated}, the model outputs $O$ may be leaked. A new reconstruction attack are proposed according to the exposed model outputs $O$ in \cite{fredrikson2015model,he2019model}, called model inversion attack with the following $L_1(d, W_k^{\calD}=O)$ loss function:
\begin{equation}\label{eq:mia-loss}
        {L}_{1}(d, W_k^{\calD}) = {L}_{1}(d, O) = C - ||\hat{O}(d) - O)||^2,
\end{equation}
in which $\hat{O}(d)$ is the output of the training model with respect to the data $d$ and $O$ is the observed output of models for the private data.

\subsection{Brute-force Attack for Encryption}
In the cryptosystem, the attack problem is formulated to restore private \textit{Key} given chosen plaintext and ciphertext (CPA). The brute-force method in a cryptosystem is to try all possible private keys until the chosen plaintext matches plaintext. Similarly, this formulation could also be written as a Bayesian inference attack with $L_1(d, W_k^{\calD})$ as: 
$$ L_1(d, W_k^{\calD})=L_1(d, (W_k^\calRO, \text{Enc}(W_k^\calRO)))=\left\{
\begin{aligned}
1, \quad \text{decrypt}(d, \text{Enc}(W_k^\calRO)) = W_k^\calRO \\
0,  \quad \text{decrypt}(d, \text{Enc}(W_k^\calRO)) \neq W_k^\calRO \\
\end{aligned}
\right.
$$
where $d$ is the estimation of private key, $W_k^{\calD} = (W_k^\calRO, \text{Enc}(W_k^\calRO))$ is pair of plaintext model information $W_k^\calRO$ and ciphertext $\text{Enc}(W_k^\calRO)$.

\subsection{Time complexity of Bayesian Inference Attack}
\label{App:time-complex}
For Gradient inversion and model inversion attack, it could be solved by the optimization loss as :
\begin{equation} \label{eq: optimizaiton}
    \max\limits_{d}L(d) = \max\limits_{d}[L_1(d, W_k^{\calD}) +L_2(d)] =\max\limits_{d}[||J(d) - w_k^{\calD}||^2 + L_2(d)],
\end{equation}
where $J$ represents the mapping from data to model gradients in gradient inversion attack, and the mapping from data to model outputs in model inversion attack. 
\begin{prop} \label{prop:time-complex-optimization}
The time complexity to solve Eq. \eqref{eq: optimizaiton} is polynomial, if any of the following conditions are satisfied:
\begin{itemize}
    \item $L$ is convex and smooth (moreover, the convergence rate is $O(\frac{1}{t})$).
    \item $L$ is convex but non-smooth, and $L$ is Lipschitz continuous (moreover, the convergence rate is $O(\frac{1}{\sqrt{t}})$).
    \item $L$ is non-convex and $L$-Smooth (moreover, the convergence rate is $O(\frac{1}{t^2})$).
\end{itemize}
\end{prop} 
\textbf{Remark:} \pref{prop:time-complex-optimization} demonstrates three conditions of $L$ that the optimization problem Eq. \eqref{eq: optimizaiton} is polynomial time. For the analysis of this proposition please refer to \cite{boyd2004convex}.\\
For Bayesian inference methods of encryption, Gentry \textit{et al.} \cite{gentry2013homomorphic} proposed a simple fully homomorphic encryption (FHE) scheme (\textbf{approximate eigenvector} method) based on the learning with errors (LWE) problem \cite{regev2009lattices}. They prove the security of proposed scheme on $GapSVP_{\gamma}$ via a classical reduction from LWE \cite{brakerski2013classical, brakerski2012fully, peikert2009public}. The best algorithms \cite{schnorr1987hierarchy, micciancio2013deterministic} for $GapSVP_{\gamma}$ requires at least $2^{\Omega(n_d/log\gamma)}$ time. The \textbf{approximate eigenvector} method reduces from $\gamma = n_d^{O(\log(n_d))}$, for which the best known algorithms run in time $2^{\Omega(n_d)}$. \textbf{Paillier} is a probabilistic asymmetric algorithm for public key cryptography based on decisional composite residuosity assumption \cite{paillier1999public}. Since it has additive homomorphic property, paillier could be applied into Fedaverage \cite{zhang2019pefl, aono2017privacy, truex2019hybrid, cheng2021secureboost} in FL, which only needs addition operations. As far as we know, the most efficient algorithm against paillier encryption is based on factoring of the a large number \cite{paillier1999public}. That is also the best known attack against RSA encryption\cite{rivest1978data}. One of the widely-known efficient algorithms for factoring the large number is general number field sieve (GNFS) \cite{buhler1993factoring}, the time complexity of which is sub-exponential ($\exp((\sqrt[3]{\frac{64}{9}} + o(1))(\ln n)^{\frac{1}{3}}(\ln \ln n)^{\frac{2}{3}})$).


%% file: Appendix/AppendixC.tex
\section{Proof for Applications of Privacy-utility trade-off}
\subsection{Proof for Randomization Mechanism} \label{sec:proof-random-app}
\hfill
\begin{lem}\label{lem:TV-gaussian}
(Total variation distance between Gaussians with the same mean \cite{devroye2018total}). Let $\mu \in R^n$, $\Sigma_1, \Sigma_2$ be diagonal matrix, and let $\lambda_1, \cdots, \lambda_n$ denote the eigenvalues of $\Sigma_1^{-1}\Sigma_2-I_n$. Then, 
\begin{equation}
    \frac{1}{100} \leq \frac{{\text{TV}}(\calN(\mu, \Sigma_1), \calN(\mu, \Sigma_2))}{\min\left\{1,\sqrt{\sum_{i=1}^n\lambda_i^2}\right\}} \leq \frac{3}{2}.
\end{equation}
\end{lem}

According to the Lemma \ref{lem:TV-gaussian}, we could estimate the range of $\gamma$:
\begin{lem} (Estimation of $\gamma$) \label{lem:gamma-random-analysis}
Let $P^{\calD}_a \sim \calN(\mu_0, \Sigma_0/K+ \Sigma_\epsilon/K)$, $P^\calRO_a \sim \calN(\mu_0, \Sigma_0/K)$, $P_k^{\calD} \sim \calN(\mu_0, \Sigma_0+ \Sigma_\epsilon)$ and $P_k^{\calRO} \sim \calN(\mu_0, \Sigma_0)$, where $\Sigma_0 = diag(\sigma_1^2, \cdots, \sigma_n^2)$ and $\Sigma = diag(\sigma_\epsilon^2, \cdots, \sigma_\epsilon^2)$  Then we have
\begin{equation}
    \frac{1}{150} \leq \gamma = \frac{{\text{TV}}(P_k^{\calRO} || P^{\calD}_k )}{{\text{TV}}(P^{\calRO}_a || P^{\calD}_a )} \leq 150.
\end{equation}
\end{lem}
\begin{proof}
According to \pref{lem:TV-gaussian}, we have
\begin{equation} \label{eq:D.2-1}
    \frac{1}{100}\min\left\{1, \sigma_\epsilon^2\sqrt{\sum_{i=1}^n\frac{1}{\sigma_i^4}} \right\} \leq {\text{TV}}(P_k^{\calRO} || P^{\calD}_k ) \leq  \frac{3}{2}\min\left\{1, \sigma_\epsilon^2\sqrt{\sum_{i=1}^n\frac{1}{\sigma_i^4}} \right\}
\end{equation}
\begin{equation} \label{eq:D.2-2}
    \frac{1}{100}\min\left\{1, \sigma_\epsilon^2\sqrt{\sum_{i=1}^n\frac{1}{\sigma_i^4}} \right\} \leq {\text{TV}}(P^{\calRO}_a || P^{\calD}_a ) \leq  \frac{3}{2}\min\left\{1, \sigma_\epsilon^2\sqrt{\sum_{i=1}^n\frac{1}{\sigma_i^4}} \right\}
\end{equation}
Combining Eq. \eqref{eq:D.2-1} and \eqref{eq:D.2-2}, we have:
\begin{equation*}
        \frac{1}{150} \leq \gamma = \frac{{\text{TV}}(P_k^{\calRO} || P^{\calD}_k )}{{\text{TV}}(P^{\calRO}_a || P^{\calD}_a )} \leq 150.
\end{equation*}
\end{proof}

\begin{lem}\label{lem: total_variation-utility-trade-off-app}
If $U(w,d)\in [0, C_4]$ for any $w \in \mathcal{W}_k$ and $d \in \calS_k$, where $k =1,\cdots, K$, then we have
\begin{align*}
    \epsilon_{u} \le C_4\cdot {\text{TV}}(P^{\calRO}_a || P^{\calD}_a).
\end{align*}
\end{lem}

\begin{proof}

Let $\mathcal U_k = \{w\in\mathcal W_k: dP_a^{\calD}(w) - dP_a^{\calRO}(w)\ge 0\}$, and $\mathcal V_k = \{w\in\mathcal W_k: dP_a^{\calD}(w) - dP_a^{\calRO}(w)< 0\}$. 
Then we have


\begin{align*}
&\epsilon_{u} = \frac{1}{K}\sum_{k=1}^K \epsilon_{u,k}\\
    & = \frac{1}{K}\sum_{k=1}^K [U_k(P^{\calRO}_{a}) - U_k(P^{\calD}_{a})]\\  
    &= \frac{1}{K}\sum_{k=1}^K \left[\mathbb E_{w\sim P^{\calRO}_{a}}[U_k(w)] - \mathbb E_{w\sim P^{\calD}_{a}}[U_k(w)]\right]\\
    & = \frac{1}{K}\sum_{k=1}^K\left[\int_{\mathcal W_k} U_k(w)dP^{\calRO}_a(w) - \int_{\mathcal W_k} U_k(w) dP^{\calD}_a(w)\right]\\
    & = \frac{1}{K}\sum_{k=1}^K\left[\int_{\mathcal{V}_k} U_k(w)[d P^{\calRO}_{a}(w) - d P^{\calD}_{a}(w)] - \int_{\mathcal{U}_k} U_k(w)[d P^{\calD}_{a}(w) - d P^{\calRO}_{a}(w)]\right]\\
    &\le \frac{C_4}{K}\sum_{k=1}^K\int_{\mathcal{V}_k} [d P^{\calRO}_{a}(w) - d P^{\calD}_{a}(w)]\\
    &=C_4\cdot {\text{TV}}(P^{\calRO}_a || P^{\calD}_a ).
\end{align*}

\end{proof}

\begin{thm}  
For the randomization mechanism adding Gaussian noise, we have the following relation between privacy leakage, utility loss and variance of Gaussian noise $\sigma_\epsilon^2$:

\begin{align} \label{tradeoff-random-app}
     C_1
    \leq \epsilon_{p} + C_3\cdot\frac{3}{2}\cdot\min\left\{1,\sigma_\epsilon^2\sqrt{\sum_{i=1}^n\frac{1}{\sigma_{i}^4}}\right\},
\end{align}
\begin{align} \label{tradeoff-random-app2}
   \epsilon_u \leq C_4 \min\left\{1,\sigma_\epsilon^2\sqrt{\sum_{i=1}^n\frac{1}{\sigma_{i}^4}}\right\},
\end{align}
where $C_1 = \frac{1}{K}\sum_{k=1}^K \sqrt{{\text{JS}}(F^{\calRO}_k || F^{\calO}_k)}$ and $C_3 = (e^{2\xi}-1)/2$ are two constants, which are independent of protection mechanism. 
\end{thm}
\begin{proof}
According to the Theorem \ref{thm: utility-privacy trade-off_JSD_mt}, we have
\begin{align*}
    C_1\le\epsilon_{p} + \frac{1}{K}\sum_{k=1}^K \frac{1}{2}(e^{2\xi}-1)\cdot {\text{TV}}(P_k^{\calRO} || P^{\calD}_k).
\end{align*}
From Lemma \ref{lem:TV-gaussian} and definitions of $P_k^{\calRO}$, $P^{\calD}_k$ , we further obtain
\begin{equation*}
     C_1\le\epsilon_{p} + C_3\cdot \min\left\{1,\sigma_\epsilon^2\sqrt{\sum_{i=1}^n\frac{1}{\sigma_{i}^4}}\right\},
\end{equation*}
where $C_3 = (e^{2\xi}-1)/2$. Therefore, Eq. \eqref{tradeoff-random-app} is proved. Moreover, combining Lemma \ref{lem:TV-gaussian} and Lemma \ref{lem: total_variation-utility-trade-off-app}, we have
\begin{align*} \label{tradeoff-random-app}
   \epsilon_u &\leq C_4\cdot {\text{TV}}(P^{\calRO}_a || P^{\calD}_a ) \\
   &\leq C_4 \min\left\{1,\sigma_\epsilon^2\sqrt{\sum_{i=1}^n\frac{1}{\sigma_{i}^4}}\right\},
\end{align*}
Therefore, Eq. \eqref{tradeoff-random-app2} is proved.
\end{proof}

\subsection{Proof for Sparsity Mechanism}\label{sec:proof-sparisty-app}
\hfill
\begin{lem} \label{lem:TV-diffmean-Gauss}
(Total variation distance between Gaussians with different means \cite{devroye2018total})
Assume that $\Sigma_1, \Sigma_2$
are positive definite, and let
\begin{equation}
    h = h(\mu_1, \Sigma_1, \mu_2, \Sigma_2) =\bigg{(}1-\frac{det(\Sigma_1)^{1/4}det(\Sigma_2)^{1/4}}{det(\frac{\Sigma_1+\Sigma_2}{2})^{1/2}}\exp\{-\frac{1}{8}(\mu_1-\mu_2)^T(\frac{\Sigma_1+\Sigma_2}{2})^{-1}(\mu_1-\mu_2)\}\bigg{)}^{1/2}
\end{equation}
Then, we have
\begin{equation}
    h^2 \leq {\text{TV}}(\calN(\mu_1, \Sigma_1), \calN(\mu_2, \Sigma_2)) \leq \sqrt{2}h.
\end{equation}
\end{lem}

\begin{thm}  \label{thm: privacy-utility-tradeoff-sparsity-app}
For the sparsity mechanism by uploading partial information to the server, suppose $P_k^{\calRO} = \calN(\mu_0, \Sigma_0) = \calN((\mu_u,\mu_o),diag(\Sigma_u, \Sigma_o))$, $P_k^{\calD} = \calN(\mu, \Sigma) = \calN((\mu_u,\mu_g),diag(\Sigma_u, \Sigma_g))$, $P_a^\calRO = \calN(\mu_0,\Sigma_0/K)$ and $P_a^{\calD} = \calN(\mu,\Sigma/K)$. We have
\begin{align}\label{eq:sparisty-privacy-app}
C_1 \leq \epsilon_p + C_3h(\mu_o, \mu_g, \Sigma_o, \Sigma_g)
\end{align}
and
\begin{align}
    \epsilon_u \leq C h(\mu_o, \mu_g, \Sigma_o, \Sigma_g),
\end{align}
where $C_1 = \frac{1}{K}\sum_{k=1}^K \sqrt{{\text{JS}}(F^{\calRO}_k || F^{\calO}_k)}, C_3 = \frac{\sqrt{2}(e^{2\xi}-1)}{2}$ are two constants independent of the protection mechanisms adopted, and $C$ is upper bound of $\sqrt{2}U(w,d)$ for any $w \in \mathcal{W}_k$ and $d \in \calS_k$.

\end{thm}

\begin{proof}
According to the \pref{lem:TV-diffmean-Gauss} and Eq. \eqref{eq: total_variation-privacy trade-off} of \pref{thm: utility-privacy trade-off_JSD_mt}, we have
\begin{align*}
    \frac{1}{K}\sum_{k=1}^K \sqrt{{\text{JS}}(F^{\calRO}_k || F^{\calO}_k)} &\leq \epsilon_p + \frac{1}{K}\sum_{k=1}^K \frac{1}{2}(e^{2\xi}-1)\cdot {\text{TV}}(P_k^{\calRO} || P^{\calD}_k)\\
    & = \epsilon_p + \frac{(e^{2\xi}-1)}{2}{\text{TV}}(P_k^{\calRO} || P^{\calD}_k)\\
    & \leq \epsilon_p + \frac{\sqrt{2}(e^{2\xi}-1)}{2}h(\mu_0, \mu, \Sigma_0, \Sigma)\\
    &= \epsilon_p + \frac{\sqrt{2}(e^{2\xi}-1)}{2}h(\mu_o, \mu_g, \Sigma_o, \Sigma_g),
\end{align*}
which could be simplified as Eq. \eqref{eq:sparisty-privacy-app},
where $C_3 = \frac{\sqrt{2}(e^{2\xi}-1)}{2}$, $C_1=\frac{1}{K}\sum_{k=1}^K \sqrt{{\text{JS}}(F^{\calRO}_k || F^{\calO}_k)}$ are two constant, which is independent of protection mechanism. Moreover, according to the \pref{lem: total_variation-utility-trade-off-app}, we have 
\begin{align*}
    \epsilon_u &\leq C_4{\text{TV}}(P_a^{\calRO} || P^{\calD}_a) \\
    & =\sqrt{2}C_4 h(\mu_0, \mu, \Sigma_0/K, \Sigma/K) = C h(\mu_o, \mu_g, \Sigma_o/K, \Sigma_g/K)\\
    & \leq C h(\mu_o, \mu_g, \Sigma_o, \Sigma_g),
\end{align*}
where $C$ is upper bound of $\sqrt{2}U(w,d)$ for any $w \in \mathcal{W}_k$ and $d \in \calS_k$. The last inequality is because $h(\mu_o, \mu_g, \Sigma_o/K, \Sigma_g/K) \leq h(\mu_o, \mu_g, \Sigma_o, \Sigma_g)$.
\end{proof}

\begin{prop} \label{prop:sparsity-prop-app}
For the sparsity mechanism by uploading partial information to the server, we have: 
\begin{itemize}
    \item The Bayesian Privacy leakage of federated system $\epsilon_{p}$ has the lower bound: $C_1 - C_2\cdot h(\mu_o, \mu_g, \Sigma_o, \Sigma_g)$, which is a increasing function of $d$.
    \item The utility loss $\epsilon_u$ has the upper bound: $ C\cdot h(\mu_o, \mu_g, \Sigma_o, \Sigma_g)$, which is an decreasing function of $d$.
\end{itemize}
\end{prop}

\begin{proof}
It just needs to prove $h(\mu_o, \mu_g, \Sigma_o, \Sigma_g)$ is decreasing function of $d$, that is, when $\mu_o \in \mathbb{R}^{n-d}$ changes to $\mu_o' \in \mathbb{R}^{n-(d-1)}$ ($d \to d-1$), $h(\mu_o, \mu_g, \Sigma_o, \Sigma_g) < h(\mu_o', \mu_g', \Sigma_o', \Sigma_g')$. Let $\mu_o' = (\mu_o, t_o), \Sigma_o' = diag(\Sigma_o, s_o^2), \mu_g' = (\mu_g, t_g), \Sigma_g' = diag(\Sigma_g, s_g^2)$. Then we have
\begin{align*}
    h(\mu_o', \mu_g', \Sigma_o', \Sigma_g') &= \bigg{(}1-\frac{det(\Sigma_o')^{1/4}det(\Sigma_g')^{1/4}}{det(\frac{\Sigma_o'+\Sigma_g'}{2})^{1/2}}\exp\{-\frac{1}{8}(\mu_o'-\mu_g')^T(\frac{\Sigma_o'+\Sigma_g'}{2})^{-1}(\mu_o'-\mu_g')\}\bigg{)}^{1/2} \\
    & = \bigg{(}1-\frac{det(\Sigma_o)^{1/4}det(\Sigma_g)^{1/4}}{det(\frac{\Sigma_o+\Sigma_g}{2})^{1/2}}\frac{(s_os_g)^{1/2}}{(\frac{s_o^2+s_g^2}{2})^{1/2}} \cdot \exp\{-\frac{1}{8}(\mu_o'-\mu_g')^T(\frac{\Sigma_o'+\Sigma_g'}{2})^{-1}(\mu_o'-\mu_g')\}\\
    & \quad  \cdot \exp\{-\frac{1}{8}(t_o-t_g)^T(\frac{s_o+s_g}{2})^{-1}(t_o-t_g)\} \bigg{)}^{1/2} \\
    & \geq \bigg{(}1-\frac{det(\Sigma_o)^{1/4}det(\Sigma_g)^{1/4}}{det(\frac{\Sigma_o+\Sigma_g}{2})^{1/2}}\exp\{-\frac{1}{8}(\mu_o-\mu_g)^T(\frac{\Sigma_o+\Sigma_g}{2})^{-1}(\mu_o-\mu_g)\}\bigg{)}^{1/2}\\
    & = h(\mu_o, \mu_g, \Sigma_o, \Sigma_g)
\end{align*}
The last equation is due to $\frac{(s_os_g)^{1/2}}{(\frac{s_o^2+s_g^2}{2})^{1/2}} \leq \frac{(s_os_g)^{1/2}}{(s_os_g)^{1/2}} =1$ and $\exp\{-\frac{1}{8}(t_o-t_g)^T(\frac{s_o+s_g}{2})^{-1}(t_o-t_g)\} \leq 1$.
\end{proof}

\subsection{Proof for Encryption Mechanism}\label{sec:proof-HE-app}
\hfill

\begin{thm}
For the encryption mechanism that encrypts the model information ($W_k^\calD = \text{Enc}(W_k^\calRO)$ using \textbf{approximate eigenvector} method), 
\begin{itemize}
    \item If the private key is unknown for server (adversaries), then $\epsilon_p = 0$ and $\epsilon_u \geq \frac{C_1}{C_2}$. 
    \item If the private key is known for server (adversaries), then $\epsilon_u = 0$ and $\epsilon_p \geq C_1$.
\end{itemize}
\end{thm}
\begin{proof}
\textit{If the private key is unknown for server (adversaries)}, since \textbf{approximate eigenvector} method \cite{gentry2013homomorphic} based on \textit{Learning With Error (LWE)} is proved to be semantic security (Lemma 5.4 in \cite{regev2009lattices}), any probabilistic, polynomial-time algorithm (PPTA) that is given the ciphertext ($W_k^\calD = \text(Enc(W_k^\calRO)$) of a certain message and the message's length, cannot determine any partial information on the message with probability non-negligibly higher than all other PPTA's that only have access to the message length. Thus we have $f_{{D_k}|{W_k}}(d|w) = f_{D_k}(d) $, for $w \in \mathcal{W}_k^\calD$. Consequently, 
\begin{align*}
    f^{\calA}_{D_k}(d) &= \int_{\mathcal{W}_k^{\calD}} f_{{D_k}|{W_k}}(d|w)dP^{\calD}_{k}(w)\\
    &=\int_{\mathcal{W}_k^{\calD}} f_{D_k}(d)dP^{\calD}_{k}(w)\\
    &= f_{D_k}(d)\int_{\mathcal{W}_k^{\calD}} dP^{\calD}_{k}(w)\\
    &=f_{D_k}(d) \\
    &= f^{\calO}_{D_k}(d).
\end{align*}
As a result, 
\begin{equation*}
f_{D_k}^{\calM}(d) = \frac{1}{2}(f^{\calA}_{D_k}(d) + f^{\calO}_{D_k}(d)) = f^{\calO}_{D_k}(d). 
\end{equation*}
Therefore, 
\begin{align*}\label{eq: def_of_pl_app}
    \epsilon_{p} &= \frac{1}{K}\sum_{k=1}^K\epsilon_{p,k} \sqrt{{\text{JS}}(F^{\calA}_k || F^{\calO}_k)}  \\
    &= \frac{1}{K}\sum_{k=1}^K\left[\frac{1}{2}\int_{\mathcal{D}_k} f^{\calA}_{D_k}(d)\log\frac{f^{\calA}_{D_k}(d)}{f^{\calM}_{D_k}(d)}\textbf{d}\mu(d) + \frac{1}{2}\int_{\mathcal{D}_k} f^{\calO}_{D_k}(d)\log\frac{f^{\calO}_{D_k}(d)}{f^{\calM}_{D_k}(d)}\textbf{d}\mu(d)\right]^{\frac{1}{2}} \\
    &= 0,
\end{align*}
Moreover, according to the Eq. \eqref{eq: total_variation-privacy trade-off_mt} of Theorem \ref{thm: utility-privacy trade-off_JSD_mt}, we have
\begin{equation*}
    \epsilon_u \geq \frac{C_1}{C_2}.
\end{equation*}

\textbf{If the private key is known for server (adversaries)}, then the server could decrypt the global model accurately so that utility loss $\epsilon_u=0$. Moreover, according to the Eq. \eqref{eq: total_variation-privacy trade-off_mt} of Theorem \ref{thm: utility-privacy trade-off_JSD_mt}, we have
\begin{equation*}
    \epsilon_p \geq C_1.
\end{equation*}

\end{proof}

\subsection{Proof for Secret Sharing Mechanism}\label{sec:proof-SecretSharing-app}
\hfill

\begin{thm}
For the secret sharing mechanism, we have that
\begin{align}
    \epsilon_{p,k} &\ge \sqrt{{\text{JS}}(F^{\calRO}_k || F^{\calO}_k)} - \frac{1}{2}(e^{2\xi}-1)\cdot\left(1 - \prod_{j = 1}^{m}\left(\frac{2\delta}{b_k^j + a_k^j}\right)\right).
\end{align}
Furthermore, we have that
\begin{align}
    \epsilon_u = 0.
\end{align}
\end{thm}
\begin{proof}

Let $W_k^{\calRO}$ represent the original model information that follows a uniform distribution over $[c_k^1 - \delta, c_k^1 + \delta]\times [c_k^2 - \delta, c_k^2 + \delta]\cdots\times [c_k^n - \delta, c_k^n + \delta]$, $W_k^{\calD}$ represent the distorted model information that follows a uniform distribution over $[c_k^1 - a_k^1, c_k^1 + b_k^1]\times [c_k^2 - a_k^2, c_k^2 + b_k^2]\cdots\times [c_k^n - a_k^{n}, c_k^n + b_k^{n}]$, and $0< \delta< a_k^{i}, b_k^{i}$, $\forall i = 1,2, \cdots, n$. Then we have that
\begin{align*}
    &\text{TV}(P^{\calRO}_k || P^{\calD}_k)\\ 
    &= \int_{[c_k^1 - \delta, c_k^1 + \delta]}\int_{[c_k^2 - \delta, c_k^2 + \delta]}\cdots\int_{[c_k^n - \delta, c_k^n + \delta]} \left(\left(\frac{1}{2\delta}\right)^{m} - \prod_{j = 1}^{m}\left(\frac{1}{b_k^j + a_k^j}\right)\right) dw_1 dw_2 \cdots d{w_{m}}\\
    & = \left(\left(\frac{1}{2\delta}\right)^{m} - \prod_{j = 1}^{m}\left(\frac{1}{b_k^j + a_k^j}\right)\right)\cdot (2\delta)^{m}.
\end{align*}

Therefore, we have that
\begin{align}
    \epsilon_{p,k} &\ge \sqrt{{\text{JS}}(F^{\calRO}_k || F^{\calO}_k)} - \frac{1}{2}(e^{2\xi}-1){\text{TV}}(P_k^{\calRO} || P^{\calD}_k)\nonumber\\
    & =  \sqrt{{\text{JS}}(F^{\calRO}_k || F^{\calO}_k)} - \frac{1}{2}(e^{2\xi}-1)\cdot\left(1 - \prod_{j = 1}^{m}\left(\frac{2\delta}{b_k^j + a_k^j}\right)\right),
\end{align}
where the first inequality is due to \pref{eq: total_variation-privacy trade-off-app} in \pref{lem: total_variation-privacy trade-off_appendix}.

For the secret sharing mechanism, the aggregated parameter does not change after being protected, which implies that $P^{\calRO}_{a} = P^{\calD}_{a}$. Therefore, we have that 
\begin{align*}
&\epsilon_{u} =\frac{1}{K}\sum_{k=1}^K \epsilon_{u,k}\\
     & = \frac{1}{K}\sum_{k=1}^K [U_k(P^{\calRO}_{a}) - U_k(P^{\calD}_{a})]\\  
    & = \frac{1}{K}\sum_{k=1}^K\left[\mathbb E_{w\sim P^{\calRO}_{a}}[U_k(w)] - \mathbb E_{w\sim P^{\calD}_{a}}[U_k(w)]\right]\\
    & = 0.
\end{align*}
\end{proof}

%% file: Appendix/AppendixD.tex
\section{Proof of the connection between BP and DP}\label{appendix:proof-BP-DP}

We may establish the connection between Differential Privacy (DP) and Bayesian Privacy (BP) as follows.  
\begin{lem}\label{lem: bp_to_dp}
Let $f_{W|D}(\cdot)$ be a privacy preserving mapping that guarantees $\xi$-maximum Bayesian privacy. That is, $\xi = \max_{w\in \mathcal{W}, d\in \mathcal{D}} \left|\log\left(\frac{f_{D|W}(d|w)}{f_{D}(d)}\right)\right|$.
Then, the mapping  $f_{W|D}(\cdot)$ is $(2\xi)$-Differentially Private:
\begin{align*}
    \frac{f_{D|W}(d|w)}{f_{D|W}(D|W')} \in [e^{-2\xi}, e^{2\xi}]. 
\end{align*}
\end{lem}

\begin{proof}
From the definition of maximum privacy leakage, we know that
\begin{align*}
    \frac{f_{D|W}(d|w)}{f_{D}(d)}\in [e^{-\xi}, e^{\xi}],
\end{align*}
for any $w\in\mathcal W$ and $d\in\mathcal D$.

Therefore, for any $w, w'\in\mathcal W$ we have
\begin{align*}
    \frac{f_{D|W}(d|w)}{f_{D|W}(d|w')} = \frac{f_{D|W}(d|w)}{f_{D}(d)}/\frac{f_{D|W}(d|w')}{f_{D}(d)}\in [e^{-2\xi}, e^{2\xi}]. 
\end{align*}
From the definition of Differential Privacy, we know that it is $(2\xi)$-Differentially Private.
\end{proof}